\documentclass[journal,letterpaper,onecolumn]{IEEEtran}
\usepackage{macros}
\usepackage{times}
\IEEEoverridecommandlockouts 

\title{Bandit algorithms to emulate human decision making using probabilistic distortions}
\author{Ravi~Kumar~Kolla,
		Prashanth~L.~A.,
		Aditya~Gopalan,
        Krishna~Jagannathan,\\
        Michael~C.~Fu
        and~Steven~I.~Marcus
        \thanks{Ravi~Kumar~Kolla is with the ABInBev, Bangalore, India, E-mail: kolla.422@gmail.com.}
        \thanks{Prashanth~L.~A. is with the Department of Computer Science Engineering, IIT Madras, Chennai, India, E-mail: prashla@cse.iitm.ac.in}
        \thanks{Aditya~Gopalan is with the Department of Electrical Communication Engineering, IISc, Bangalore, India, E-mail: aditya@iisc.ac.in}
        \thanks{Krishna~Jagannathan is with the Department of Electrical Engineering, IIT Madras, Chennai, India, E-mail: krishnaj@ee.iitm.ac.in}
        \thanks{Michael~C.~Fu is with the Robert H. Smith School of Business \& Institute for Systems Research, University of Maryland, College Park, MD, USA, E-mail: mfu@isr.umd.edu}
        \thanks{Steven~I.~Marcus is with the Department of Electrical and Computer Engineering \& Institute for Systems Research, University of Maryland, College Park, MD, USA, E-mail: marcus@umd.edu}
        \thanks{The material in this paper was presented in part at the 2017 AAAI Conference on Artificial Intelligence \cite{aditya2016weighted}.}
        \thanks{This work was supported in part by the U.S. Air Force Office of Scientific Research (AFOSR) under Grant FA9550-20-1-0211, and by a DST grant under the ECRA program.}
}

\begin{document}

\maketitle

\begin{abstract}
Motivated by models of human decision making proposed to explain commonly observed deviations from conventional expected value preferences, we formulate two stochastic multi-armed bandit problems with distorted probabilities on the reward distributions: the classic $K$-armed bandit and the linearly parameterized bandit settings. We consider the aforementioned problems in the regret minimization as well as best arm identification framework for multi-armed bandits. For the regret minimization setting in $K$-armed as well as linear bandit problems, we propose algorithms that are inspired by Upper Confidence Bound (UCB) algorithms, incorporate reward distortions, and exhibit sublinear regret. 
For the $K$-armed bandit setting, we derive an upper bound on the expected regret for our proposed algorithm, and then we prove a matching lower bound to establish the order-optimality of our algorithm. For the linearly parameterized setting, our algorithm achieves a regret upper bound that is of the same order as that of regular linear bandit algorithm called Optimism in the Face of Uncertainty Linear (OFUL) bandit algorithm, and unlike OFUL, our algorithm handles distortions and an arm-dependent noise model. For the best arm identification problem in the $K$-armed bandit setting, we propose algorithms, derive guarantees on their performance, and also show that these algorithms are order optimal by proving matching fundamental limits on performance. For best arm identification in linear bandits, we propose an algorithm and establish sample complexity guarantees. Finally, we present simulation experiments which demonstrate the advantages resulting from using distortion-aware learning algorithms in a vehicular traffic routing application. 

\end{abstract}


\section{Introduction}
Traditional bandit approaches to optimization are based on the central assumption that the expected value of an arm is a good measure of its utility, and hence, finding an arm with the highest expected value is a reasonable goal for a bandit algorithm. This assumption, however, may not hold in human-centric decision making systems, as humans are known to be prone to various emotional and cognitive biases. For instance, given a choice between an investment that gives $\$10$ surely and another that returns $\$10000$ with probability (w.p.) $0.001$, a majority of human subjects  prefer the latter. However, turning the same example into losses, i.e., an investment that loses $\$10$ surely vs. one that loses $\$10000$ w.p. $0.001$, the human preference would change to pick the former. Observe that in both examples, even though the expected values of both are the same, human preferences change dramatically from the profit to the loss setup. For an in-depth understanding of the shortcomings of using expected value to explain human preferences, the reader is referred to the well-known Allais and Ellsberg paradoxes given in \cite{allais53, ellsberg61}.

Violations of the expected value-based preferences in human-based decision making systems can be alleviated by incorporating distortions in the underlying probabilities of the system~\cite{starmer2000developments}. 
Models involving probabilistic distortions have a long history in behavioural science and economics, and we bring this idea to multi-armed bandit problems. In particular, we base our approach on rank-dependent expected utility (RDEU) \cite{quiggin2012generalized}, which includes  \textit{cumulative prospect theory} (CPT) considered by the authors in~\cite{kahneman1979prospect}\footnote{Daniel Kahneman was awarded the 2002 Nobel prize in economics for his work on prospect theory and human biases under uncertainty.}. CPT is an enhancement of the influential \textit{prospect theory} and incorporates RDEU-style distortion of probabilities. The latter distortion ensures that there is no violation of first-order stochastic dominance -- a shortcoming  of prospect theory.
CPT has been shown through experiments involving human subjects,  to model human preferences well by psychologists and behavioral economists \cite{starmer2000developments}, \cite[Chapter~4]{quiggin2012generalized}. 

\begin{figure}[t]
\centering
\tabl{c}{
  \scalebox{0.85}{\begin{tikzpicture}
  \begin{axis}[width=10cm,height=6cm,legend pos=south east,
           grid = major,
           grid style={dashed, gray!30},
           xmin=0,     
           xmax=1,    
           ymin=0,     
           ymax=1,   
           axis background/.style={fill=white},
           ylabel={\large Weight $\bm{w(p)}$},
           xlabel={\large Probability $\bm{p}$}
           ] 
          \addplot[domain=0:1, blue, thick,smooth,samples=1500] 
             {pow(x,0.69)/pow((pow(x,0.69) + pow(1-x,0.69)),1.44)}; 
             \node at (axis cs:  0.8,0.35) (a1) { $\bm{\frac{p^{0.61}}{(p^{0.61}+ (1-p)^{0.61})^{\frac1{0.61}}}}$};           
             \draw[->] (a1) -- (axis cs:  0.7,0.58);
                 \addplot[domain=0:1, black, thick, dashed]           {x};                      
  \end{axis}
  \end{tikzpicture}}\\[1ex]
}
\caption{Graphical illustration of a typical weight function that inflates low probabilities and deflates large probabilities. The weight function used in the figure is the one recommended in \cite{tversky1992advances} based on empirical tests involving human subjects.
}
\label{fig:weight}
\end{figure}
The distortions happen via a weight function $w:[0,1] \rightarrow [0,1]$ that transforms probabilities in a nonlinear fashion. As illustrated in Figure \ref{fig:weight}, a typical weight function, say $w:[0,1]\rightarrow [0,1]$, has an inverted-S shape. In other words, $w$ inflates low probabilities and deflates high probabilities and can explain human preferences well. For instance, in the example above, if we choose $w(1/{10}^6) > 1/{10}^6$ and take expectations w.r.t. the $w$-distorted distribution, then Arm 1 would be preferable when the problem is set up with rewards and Arm 2 for the problem with costs. The suitability of this approach, especially with an inverted-S shaped weight function, to model human decision making (and thus preferences) has been widely documented in
\cite{prelec1998probability,wu1996curvature,conlisk1989three,camerer1989experimental,cherny2009new,gonzalez1999shape}.

We use a traveller's route choice as a running example to illustrate the main ideas in this paper. The setting here is that a human travels from a source (e.g., home) to a destination (e.g., office), both fixed, every day. He/she has multiple routes to choose from, and travelling each route incurs a stochastic delay with unknown distributions. The problem then is to choose a route that minimizes some function of delay. Using expected delay may not lead to a routing choice that is most appealing to the human traveller. In addition to the examples mentioned earlier, intuitively, humans could prefer a route with a slight excess of delay over another that has a small probability of getting into a traffic jam that takes hours to be resolved.  It is imperative that an automated route recommendation algorithm, running for instance as an application on the traveller's mobile device, factor in this distortion in its effort to learn the most preferred route in an online manner.
We treat this problem in two regimes: first, a setting where the number of routes or decisions is small, so the traveller can afford to try each of the routes a small number of times before fixing on the ``best'' route; second, a big road network setting involving a large number of routes or decisions, in which trying out all possible routes may be prohibitively expensive. 

We formalize two probabilistically distorted bandit settings that correspond to the two routing setups mentioned above. The first is the classic $K$-armed bandit setting, while the second is the linear bandit setting. In both settings, we define the weight-distorted reward $\mu_x$ for any arm $x$ in the space of arms $\X$ as follows:
\begin{align}
\hspace{-0.4em}\mu_x \!=\! \intinfinity w(\prob{\arm_x > z}) dz \!-\! \intinfinity w(\prob{-\arm_x > z}) dz,\label{eq:cpt-intro}
\end{align}
\noindent where $w: [0, 1] \rightarrow [0, 1]$ is the weight function that satisfies $w(0)=0$ and $w(1)=1,$ and $\arm_x$ is the random variable (r.v.) corresponding to the stochastic rewards from arm $x \in \X$. Note that by choosing the identity weight function $w(p)=p$, we obtain the standard expected reward function, $\mu_x = \E{ (\arm_x^+) } - \E{ (\arm_x^-) } = \E(\arm_x)$, where $y^+ = \max(y,0)$ and  $y^- = \max(-y,0)$ denote the positive and negative parts of $y\in \R$, respectively. Thus, $\mu_x$ as in \eqref{eq:cpt-intro} generalizes standard expected value. As discussed earlier, $w$ has to be chosen in a non-linear fashion, to capture human preferences, a modelling approach which has strong empirical support. 

The goal is to find an arm $x_*$ that maximizes \eqref{eq:cpt-intro}; we consider this problem in two popular bandit learning frameworks - regret minimization and best arm identification. The former setting captures the popular exploration-exploitation dilemma, while the latter is motivated through simulation optimization. In the following, we outline the challenges involved in finding the optimal arm $x_*$ in both frameworks and summarize our contributions.  

\subsection{Regret minimization with weight-based distortions}
The goal here is to  accumulate a high (weight-distorted) reward by sampling from the arms' distributions and quickly discerning (and discarding) the sub-optimal arms. Such an objective makes sense in the traffic routing example described earlier. To elaborate, in the routing example, a mobile application recommending routes to the human user would be beneficial only if it does not explore every possible route extensively and instead, finds the most-appealing route (quantified by the associated weight-distorted value) as soon as possible.
 
The problem of weight-distorted reward optimization in a regret minimization framework is challenging because the existing bandit solutions, for instance, the popular UCB algorithm, cannot handle distortions. This is because the environment provides samples from the distribution $F_x$ when arm $x$ is pulled, while the integral in \eqref{eq:cpt-intro} involves a distorted distribution. The implication is that a simple sample mean and a confidence term suggested by the Hoeffding inequality is not enough to derive the UCB values for any arm in the weight-distorted reward setting that we consider. This is because, one requires a good enough estimate of $F_x$ to estimate $\mu_x$. To illustrate, a ($\alpha$-H\"{o}lder-continuous) weight function such as $w(t) := t^\alpha$, when applied to a Bernoulli($p$) distribution, distorts the mean to 
$p^\alpha$ from $p$, and can introduce an arbitrarily large scaling for arms with real expectations close to $0$. It follows that nonlinear weight distortion can, in fact, change the order of the optimal arm, resulting in a distortion-unaware algorithm like UCB converging to the {\em wrong} arm and incurring linear regret. The W-UCB algorithm that we propose incorporates an empirical distribution-based approach, similar to the one in \cite{prashanth2015cumulative}, to estimate $\mu_x$. However, unlike the latter, our algorithm incorporates a confidence term relying on the Dvoretzky-Kiefer-Wolfowitz (DKW) inequality \cite{massart1990tight, wasserman2006} that ensures the W-UCB values are a high-probability bound on the true value $\mu_x$.    
We provide upper bounds on the regret of W-UCB, assuming $w$ is H\"{o}lder-continuous. It is worth noting that concentration inequalities for the empirical CDF have previously been used in multi-armed bandits~\cite{salomon2011deviations} for achieving a different goal than the goal considered in this work. 

Next, we consider a linear bandit setting with weight distortions that can be motivated as follows:
Consider a network graph $G=(V,E)$, $|E| = d$, with a source $s\in V$ and destination $t \in V$. The interaction proceeds over multiple rounds, where in each round $m$, the user picks a route $x_m$ from $s$ to $t$ (a route is a collection of edges encoded by a vector of $0-1$ values in $d$ dimensions) and experiences a stochastic delay $x_m\tr\left(\theta + N_m\right)$. Here $\theta \in \R^d$ is an underlying model parameter and $N_m \in \R^d$ is a random noise vector, both unknown to the learner. The physical interpretation is that the nodes represent geographical locations (say junctions), and the edges are roads that connect nodes. An edge from $i$ to $j$ will have an edge weight $\theta_{ij}$, which quantifies the delay for this edge. 

Notice that the observations include a noise component that scales with the route chosen -- this is unlike the model followed in earlier linear bandit works \cite{abbasi2011improved,dani2008stochastic}, where the noise was independent of the arm chosen. Our noise model makes practical sense, because the observed delay in a road traffic network depends on the length of the route, e.g., one would expect more noise in a detour involving ten roads than in a direct one-road route.
The aim is to find a low-delay route that satisfies the user.  The setting is such that the number of routes is large (so the regular $K$-armed bandits don't scale), and one needs to utilize the linearity in the costs (delays) to find the optimal route, where optimality is qualified in 
terms of a weight-distorted expectation.

For the linear bandit setting, we propose a variant of the OFUL algorithm in \cite{abbasi2011improved} that incorporates weight-distorted values in the arm selection step. The regret analysis of the resulting W-OFUL algorithm poses novel challenges compared to that in the linear bandit problem, primarily because the instantaneous reward (and hence regret) at each round is in fact a {\em nonlinear} function of the features of the played arm. This occurs due to the distortion of probabilities by the weight function, although the actual observation (e.g., network delay in the example above) is linear in expectation over the played arm's features. The weight function can not only change the optimal arm but can also potentially amplify small differences in real expected values of arms to much larger values, leading to a blowing up of overall regret. 
However, our analysis shows that regret in weight-distorted reward as the performance metric can be controlled at the same rate as that in standard linear bandit models, irrespective of the structure of the distortion function. More specifically, we show, using a careful analysis of the effect of weight distortion, that the regret of the W-OFUL algorithm is no more than $O\left(d \sqrt{n} \mbox{ polylog}(n)\right)$ in $n$ rounds with high probability, similar to the guarantee enjoyed by the OFUL algorithm in linear bandits (note however that the identity of the optimal arm may be different due to weight distortion in rewards). We conduct simulation experiments in a vehicular traffic routing application, and our results indicate that the regular expected delay minimizing bandit algorithm in \cite{abbasi2011improved} converges to a route that has the lowest mean delay, but with a small chance of inordinately high delay, while our algorithm converges to a route that has slightly higher mean delay, but no chance of very high delays. 

\subsection{Best arm identification with with weight-based distortions}
The goal here is to find the arm that has the highest  weight-distorted reward defined in~\eqref{eq:cpt-intro} and hence, is best aligned with human preferences. However, unlike the regret minimization setting outlined earlier, here the requirement is for an algorithm that is allowed to explore the arms freely, as long as it outputs the optimal weight-distorted valued arm with high confidence using the minimum number of samples. 

Simulation optimization \cite{fu2015handbook} provides a strong practical motivation for solving the best-arm identification with weight-based distortions. Several human-centric real-world systems in  disciplines such as networks, healthcare, finance, are too complex to directly optimize among a set of choices. A viable alternative is to build a simulator for various components of the system, and then perform the optimization over decisions or choices via simulator access. Simulation optimization refers to this setting, where the goal is to find the optimum choice for a certain design parameter. For a given parametric description of the system, performance evaluations using the simulator are typically \textit{noisy} (i.e., have a spread or distribution), and each simulation to obtain an evaluation is computationally expensive. Thus, in addition to searching for optima, a simulation optimization algorithm has to ensure that the number of evaluations is minimum. 

Best arm identification with weight-based distortions fits perfectly well into a simulation optimization context, as several human-based decision making systems can be optimized using sophisticated simulators. For instance, VISSIM/CORSIM are popular traffic simulators, which can be employed to optimize the route so that the delay perceived by the road users (humans) is minimum. The CPT perspective here is that a route that has the lowest average delay but has a small probability of getting badly congested may not be as appealing to humans when compared to, say, an alternate route that has a slightly higher average delay but no chance of bad congestion. 
 
As in the regret minimization setting, we consider the best-arm identification problem in two regimes: (i) when the number of arms $K$ is small - this is the so-called (classical) $K$-armed bandit setting; and (ii) when $K$ is large, but the underlying rewards follow a linear model with an unknown parameter.
For the $K$-armed bandit problem under the fixed confidence objective\footnote{In the fixed confidence setting, the bandit algorithm's goal is to find, with high confidence, the arm with the highest weight-distorted reward.}, we propose W-LUCB algorithm, a non-trivial adaptation of LUCB algorithm in \cite{kalyanakrishnan2012pac}, and then establish an upper bound on its expected sample complexity and also exhibit a matching lower bound.
Furthermore, for $K$-armed bandits under the fixed budget objective\footnote{In the fixed budget setting, the bandit algorithm is given a budget of $n$ samples and the goal is to recommend the optimal arm after exhausting the budget, while minimizing the probability of an erroneous recommendation.}, we propose W-SR algorithm which is an adaptation of the successive rejects algorithm in \cite{audibert2010best} to incorporate weight-based distortions, provide a theoretical upper bound  on the probability of returning a sub-optimal arm at the end of the sampling budget and finally, prove matching  lower bounds to establish order-optimality of our proposed algorithm.
For the  linear bandits setting with fixed confidence, we propose W-G algorithm inspired from the work in \cite{soare2014best} and derive an upper bound on the sample complexity of our proposed algorithm. 

 In the following table, we list all proposed algorithms under their respective settings considered in the paper.
\begin{table}[h]
	\caption{All algorithms under various settings considered in the paper.}
	\label{tab:algorithms_list}
	\centering
	\begin{tabular}{|c|c|c|}
		\toprule 
		\textbf{Bandit type}& \textbf{Objective} & \textbf{Algorithm} \\\midrule
		K-armed & Regret minimization & W-UCB\\\midrule
		K-armed & Best arm identification under fixed confidence & W-LUCB\\\midrule
		K-armed & Best arm identification under fixed budget & W-SR\\\midrule
		Linear & Regret minimization & W-OFUL\\\midrule
		Linear & Best arm identification under fixed confidence & W-G\\\midrule
	\end{tabular}
	\end{table}

\subsection{Related work}
A major line of work in multi-armed bandits concerns the exploration-exploitation dilemma for which the UCB algorithm in \cite{auer2002finite} is a popular solution. Best-arm identification problem, considered in \cite{audibert2010best,even2002pac,mannor2004sample, russo2016simple}, provides an alternative formulation where the focus is on exploration. In \cite{russo2016simple}, the authors consider the best arm identification problem in stochastic multi-armed bandits from a Bayesian perspective as opposed to the frequentist approach and weight-distortion setup used in this work. Linearly parameterized bandits have been studied under regret minimization \cite{dani2008stochastic,abbasi2011improved} as well as best-arm identification \cite{soare2014best, jedra2020optimal} settings. In the recent work \cite{jedra2020optimal} on best arm identification in linear bandits, the authors propose and study the sample complexity of an algorithm that is inspired by the track-and-stop approach to best arm identification. The authors establish asymptotic optimality of their algorithm. In contrast, we consider the G-allocation algorithm from~\cite{soare2014best} and devise a weight-distorted variant. In terms of theoretical guarantees, we provide finite sample guarantees for the proposed algorithm. For the regret minimization as well as best arm identification problems, previous works have mostly considered the problem of identifying an arm with the highest mean~\cite{auer2002finite, dani2008stochastic, abbasi2011improved, mannor2004sample, even2006action, kalyanakrishnan2012pac, audibert2010best, kaufmann2015complexity, soare2014best}. To the best of our knowledge, there is no prior work in the literature that deals with the problem of identifying an arm with the highest weight-distorted reward in a multi-armed bandit setting.

Variance as a risk measure has been explored in \cite{sani2012risk}, while coherent risk measures are considered in \cite{maillard2013robust}. Conditional value at risk, which is a popular risk measure that is coherent, has been explored in a bandit context in \cite{galichet2013exploration}.  A UCB variant that incorporates variance in the confidence term is UCB-V --- see \cite{AudMunSze09:UCBV}. The aforementioned algorithms cannot be used as surrogates for optimizing the weight-distorted reward due to the inherent differences between coherent risk measures and weight-based distortions, which are non-coherent.
More recently, in \cite{cassel2018general}, the authors derive a general framework for incorporating risk measures based on the empirical distribution function. This work can handle risk metrics such as conditional value at risk (CVaR), but, as we show later, risk measures based on weight-distortions cannot, in general, be fit into this framework. Beyond bandit learning, a risk measure based on cumulative prospect theory has been explored in a supervised learning context in \cite{leqi2019human,khim2020uniform}. 

The closest related previous contribution is that of \cite{prashanth2015cumulative}, where the authors bring in ideas from CPT to a reinforcement learning (RL) setting. In contrast, we formulate two multi-armed bandit models that incorporate weight-distortions. From a theoretical standpoint, we handle the exploration-exploitation tradeoff and pure exploration bandit search problem, while the focus of \cite{prashanth2015cumulative} was to devise a policy-gradient scheme, given biased estimates of a certain CPT-value defined for each policy. Moreover, we provide finite-time bounds for both bandit settings, while the guarantees for the policy gradient algorithm in~\cite{prashanth2015cumulative} are asymptotic in nature.

Previous works involving RDEU and CPT are huge in number; at a conceptual level, the work in this paper integrates machine learning (esp. bandit learning) with an RDEU approach that involves a probabilistic distortion via a weight function. To the best of our knowledge, RDEU/CPT papers in the literature assume model information, i.e., a setting where the distributions of the arms are known, while we have a \textit{model-free} setting where one can only obtain sample values from the arms' distributions. Our setting makes practical sense; for instance, in the traveller's route choice problem one can only obtain sample delays for a particular route, while the distribution governing the delays for any route is not known explicitly.  

The rest of the paper is organized as follows: 
The model and the algorithms for $K$-armed bandits are presented in Section~\ref{sec:model-karmed} and ~\ref{sec:algos-karmed}, respectively. The model and the algorithms for linear bandits are presented in Section~\ref{sec:model-linear} and~\ref{sec:algos-linear}, respectively.  Section~\ref{sec:experiments-regret} presents the numerical results for the regret minimization setting under linear bandits. We provide the conclusions in Section~\ref{sec:conclusions}. 
Detailed proofs of analytical results pertaining to $K$-armed and linear bandits are given in Appendices~\ref{sec:proofs-karmed} and~\ref{sec:proofs-linear}, respectively, in the supplementary material. Additional numerical experiments for regret minimization and best arm identification settings are given in Appendices~\ref{sec:experiments-regret-appendix} and~\ref{sec:experiments-bestarm}, respectively, in the supplementary material.

\section{K-armed bandit models}
\label{sec:model-karmed}
We consider the stochastic multi-armed bandit setting with a finite set of arms \\$\A = \lbrace 1, 2, \cdots, K \rbrace$.  Each arm $k$ is associated with a distribution $F_k$, which is unknown a priori to the agent. Let $\mu_k$ denote the weight-distorted reward of arm $k$ and $a^* = \argmax_k~\mu_k$ denote the optimal arm. Note that the quantity $\mu_k$, defined in \eqref{eq:cpt-intro}, can be seen to be equivalent to the following:
\begin{equation}
\mu_k  := \intinfinity w(1-F_k(z)) dz - \intinfinity w(F_k(-z)) dz, \label{eq:cpt-general}
\end{equation}
where $w:[0,1] \rightarrow [0,1]$ satisfies $w(0) =0$ and $w(1) = 1,$ as described earlier, is a weight function that distorts probabilities. The optimal arm is the one that maximizes the weight-distorted reward, i.e., 
$\mu_* = \max_k \mu_k.$ The optimal arm is not necessarily unique, i.e., there may exist multiple arms with the optimal weight-distorted reward $\mu_*$. For sake of simplicity, we assume that $\mu_1 > \mu_2 \geq \cdots \geq \mu_K$, i.e., arm 1 is the arm with the highest weight-distorted reward, which we call the unique optimal arm~\footnote{The ordering assumption is without loss of generality, and without uniqueness the problem can only get easier.}. For $k \geq 2,$ let $\Delta_k = \mu_1 - \mu_k$ denote the gap between the weight-distorted rewards of the optimal arm and an arm $k,$ and $\Delta_1 = \mu_1 - \mu_2\footnote{Generally, it is considered that $\Delta_1 = 0$ in the regret minimization literature. However, we require this notation, $\Delta_1 = \mu_1 - \mu_2,$  for the setting of best arm identification.}.$ 

In the following sections, we present the learning model and objective for $K$-armed bandits under the regret minimization and the best arm identification frameworks. 

\subsection{Regret minimization with weight-based distortions}
\label{sec:model-karmed-regret}
In each round $m=1,\ldots,n$, the algorithm pulls an arm $I_m \in \A = \{1,\ldots,K\}$ and obtains a sample reward from the distribution $F_{I_m}$ of arm $I_m$. The classic objective is to play (or pull) the arm whose expected reward is the highest. In this paper, we take a different approach inspired by non-expected utility  models and use the weight distorted-reward $\mu_k$ as the performance criterion for any arm $k$. 

We now define the cumulative regret $R_n$ as follows:
\begin{align}
R_n = n \mu_1 - \sum_{k=1}^K  T_k(n) \mu_k\, ,\label{eq:regret-karmed} 
\end{align}
where $T_k(n)= \sum_{m=1}^n I(I_m=k)$ is the number of times arm $k$ is pulled up to time $n$. The expected regret can be written as 
$\E R_n = \sum_{k=2}^K \E[T_k(n)] \Delta_k.$ Note that this definition of regret arises from the interpretation that each arm $k$ is associated with a deterministic value $\mu_k$. The highest possible cumulative reward that can be accumulated in $n$ rounds is thus $n \mu_1$, while that accumulated by a given strategy is $\sum_{k=1}^K T_k(n) \mu_k$. Thus, the regret as defined above is a measure of  the rate at which a strategy converges to playing the optimal arm in the sense of weighted or distorted reward. Our goal is to design algorithms that have the lowest expected regret.

We remark that the regret performance measure, as defined above, is explicitly defined within a stochastic model for rewards. Thus, low-regret algorithms designed for the non-stochastic setting, e.g., EXP3~\cite{Auer02NonSto}, are not inherently suitable for this problem, as they do not factor in the distortion caused in (expected) reward. A similar observation holds for conventional stochastic bandit algorithms such as UCB, and algorithms sensitive to arm reward variances such as UCB-V~\cite{AudMunSze09:UCBV} --  once weight distortion is incorporated, the algorithm can converge to an arm that is not weight-distorted value optimal. Thus, applying a variance-sensitive algorithm (like UCB-V) may still yield linear regret in the distorted setting.
\subsection{Best arm identification with weight-based distortions}
\label{sec:model-karmed-bestarm}
The bandit algorithm interacts with the environment as follows: 
In each round $m=1,2,\ldots$, the algorithm chooses an arm $I_m \in \{1,\ldots,K\}$ and receives an independent sample from the distribution corresponding to arm $I_m$.
The objective is to find the arm with the highest weight-distorted value by sequentially sampling the arms; we study this problem under the following popular variants:

\textbf{Fixed Confidence:} Here, a confidence parameter $\delta \in (0,1)$ is given. At each round, an algorithm can either (a) stop and declare an arm as optimal, or (b) continue and choose an arm to sample. The algorithm must ensure that the arm declared upon stopping is indeed optimal with a probability of at least $(1 - \delta).$ An algorithm's performance is quantified by the (expected) number of samples it uses before it stops, which we call the sample complexity of the algorithm.

\textbf{Fixed Budget:} Here, a budget of $n$ rounds is given, with the stipulation that the algorithm must declare an arm as optimal at the end of $n$ rounds. The algorithm's goal is to minimize the probability of declaring a sub-optimal arm as optimal.

Our goal is to design algorithms that take the lowest number of rounds and have the smallest probability of error under fixed confidence and fixed budget settings, respectively.

\section{K-armed bandit algorithms}
\label{sec:algos-karmed}
In this section, we first introduce an estimator for weight-distorted reward of an arm, given a set of i.i.d. samples drawn from the associated arm's distribution and a weight function. Then, we present algorithms for $K$-armed bandits under regret minimization and best arm identification settings.  
\subsection{Estimating the weight-distorted reward}
Estimating the weight-distorted reward of any arm $k$ is non-trivial; one cannot just use a Monte Carlo approach with sample means, because weight-distorted reward involves a distorted distribution, whereas the samples come from the undistorted distribution $F_k$. Thus, one needs to estimate the entire distribution, and, for this purpose, we adapt the quantile-based approach, originally proposed in~\cite{prashanth2015cumulative}.  For convenience, we briefly recall this estimation scheme for weight-distorted reward of an arm $k$ in Algorithm~\ref{alg:cpt-estimator}.
\begin{algorithm}  
\caption{Weight-distorted reward estimator}
\label{alg:cpt-estimator}
\begin{algorithmic}
\State {\bfseries Input:} $l$ i.i.d samples of arm $k$, $Y_{k,1},\ldots, Y_{k,l},$ weight function $w$.

\vspace{1ex}

\State Sort the samples in ascending order as follows:\\ $Y_{[k, 1]} \leq Y_{[k, 2]} \leq \dots Y_{[k, l_b]} \leq 0 \leq Y_{[k, l_b+1]} \leq \dots \leq Y_{[k, l]},$ where $l_b \in \lbrace 0, 1, 2, \dots, l \rbrace.$   

\vspace{1ex}

\State Let $ \widehat{\mu}_{k, l}^+  := \sum\limits_{i=l_b+1}^l Y_{[k,i]}\left[ w \left( \frac{l+1-i}{l} \right) - w \left( \frac{l-i}{l} \right) \right]$ and  
$\widehat{\mu}_{k, l}^-  := \sum\limits_{i=1}^{l_b} Y_{[k, i]}\left[ w \left( \frac{i - 1}{l} \right) - w \left( \frac{i}{l} \right) \right].$ 

\vspace{1ex}

 \State Set $\widehat{\mu}_{k, l} = \hat{\mu}_{k, l}^+ - \hat{\mu}_{k, l}^-.$

\vspace{1ex}

\State {\bfseries Output:} Return $\hat{\mu}_{k, l}$.
\end{algorithmic}
\end{algorithm}

It can be seen that $\widehat \mu_{k,l}$ is the weight-distorted reward of the empirical distribution of samples from arm $k$ seen thus far, i.e.,
\begin{align*}
\widehat \mu_{k,l}  &:= \intinfinity w(1-\hat F_{k,l}(z)) dz - \intinfinity w(\hat F_{k,l}(-z)) dz, \label{eq:cpt-est-equiv}
\end{align*}
where $\hat F_{k,l}(x):=\frac{1}{l} \sum_{i=1}^{l} I_{\left[Y_{k,i} \leq x\right]}$ denotes the empirical estimate of the distribution $F_k.$ In particular, the first and second integrals above correspond to $\widehat{\mu}_{k, l}^+$ and $\widehat{\mu}_{k, l}^-,$ respectively. 

We next provide a sample complexity result for the accuracy of the estimator $\widehat \mu_{k,l}$ under the following assumptions:\\
\noindent\textbf{(A1)}  
The weight function $w : [0, 1] \rightarrow [0, 1]$ with $w(0) = 0$ and $w(1) =1$, and is H\"{o}lder-continuous with constant $L > 0$ and exponent $\alpha \in (0,1]$, i.e., $\sup\limits_{x \neq y} \frac{| w(x) - w(y) |}{| x-y |^{\alpha}} \leq L.$ \\
\noindent\textbf{(A2)}  
The arms' rewards are bounded by $M > 0$ almost surely.

(A1) is necessary to ensure that the weight-distorted reward $\mu_k,$ for $k \in \{ 1,\ldots,K \}$ is finite. Moreover, the popular choice for the weight function, proposed in~\cite{tversky1992advances} and illustrated in Figure \ref{fig:weight}, is H\"{o}lder-continuous.
\begin{theorem}[\textit{Sample complexity of estimating weight-distorted reward}]
\label{thm:cpt-est}
\ \\
Assume (A1)-(A2). Then, for any $\epsilon >0$ and any $k\in \{1,\ldots,K\}$, we have 
\begin{equation*}
\mathbb{P}(\left |\widehat \mu_{k,m} - \mu_k \right| > \epsilon ) \le 2\exp\left(-2m(\epsilon/LM)^{2/\alpha}\right).
\end{equation*}
\end{theorem}
\begin{proof}
The proof is very similar to that of~\cite[Proposition 2]{prashanth2015cumulative}. The proof relies on the Dvoretzky-Kiefer-Wolfowitz (DKW) inequality, which gives finite-sample exponential concentration of the empirical distribution $\hat{F}_{k,m}$ around the true distribution $F_k$, as measured by the $||\cdot||_\infty$ function norm \cite[Chapter 2]{wasserman2006} and~\cite{massart1990tight}.
\end{proof}
\begin{remark}
Weight-distorted reward estimation is easier when the form of the underlying distribution is known, as illustrated by the following example: Consider a Bernoulli distribution with parameter $p$.
The weight-distorted reward for this distribution is $w(p)$, and the latter can be estimated as $w(\hat{p})$, where $\hat{p}$ is an estimate of $p$. 
\end{remark}

\begin{remark}
For the special case of distributions with bounded support, using DKW inequality, together with the H\"{o}lder-continuity of the weight function, allowed the derivation of the concentration result in Theorem~\ref{thm:cpt-est}. However, we do not see this argument work for distributions with unbounded support, e.g., sub-Gaussian.
For the case of sub-Gaussian distributions, an approach that works uses truncation to arrive at a concentration result, which is given in the below proposition.
\end{remark}

Suppose the underlying distribution is $\sigma$-sub-Gaussian. We define the following truncated estimator for weight-distorted value as follows:
\begin{align*}
\mu_n = \int_{0}^{M_n}w(1-F_{n}(z))\mathrm{d}z - \int_{0}^{M_n}w( F_{n}(-z))\mathrm{d}z,
\end{align*}
where $M_n = \sigma\sqrt{\log n}$.

We now present a concentration bound for the truncated estimator defined above.
\begin{proposition}
\label{prop:sub-gaussian-concentration}
	Let $Y$ be a $\sigma$-sub-Gaussian r.v. Let $\mu(Y)$ denote the weight-distorted value associated with r.v. $Y$, and let $\mu_n$ be the truncated estimator, formed from $n$ i.i.d. samples of the distribution underlying $Y$. Then, for any $\epsilon >  \frac{8L}{\alpha n^{\alpha/2}}$, we have
	\begin{align*}
	\Prob{\mu_n - \mu(Y) > \epsilon} \le  2\exp\left( -2 n \left(\frac{2}{L^2\log n}\right)^{\frac1{\alpha}} \left(\epsilon - \frac{8L}{\alpha n^{\alpha/2}} \right)^{\frac{2}{\alpha}}  \right).
	\end{align*}
\end{proposition}
\begin{proof}
See Appendix~\ref{sec:subgauss-conc}.
\end{proof}

\begin{remark}
In Proposition~3 of~\cite{prashanth2017cptlong}, the authors provide a tail bound of the order \\ $\left(2n e^{-n^\frac{\alpha}{2+\alpha}} 
+ 2 e^{-n^{\frac{\alpha}{2+\alpha}}\left(\frac{\epsilon}{2L}\right)^{\frac{2}{\alpha}}}\right)$ for weight-distorted value estimation in the sub-Gaussian case, for all $n \geq \left(\frac{\ln2-\ln\epsilon}{2\alpha}\right)^{\alpha+2}$. In contrast, the tail bound provided in the above proposition is of the order $\left(2 e^{-c n \left(\epsilon - \frac{8L}{\alpha n^{\alpha/2}}\right)^{\frac{2}{\alpha}}}\right)$ holds for all $n\ge 1$. Our bound exhibits improved dependence on the number of samples $n$, while having a similar dependence on the accuracy $\epsilon$. 
\end{remark}

\subsection{Regret minimization with weight-based distortions}
\label{sec:algos-karmed-regret}
We now propose and study algorithms for $K$-armed bandits under the regret minimization setting. 
Recall, from (A1) and (A2), that  $\alpha \in (0, 1]$ denote the \holder exponent, $L > 0$ the \holder constant and $M > 0$ the bound on the stochastic rewards from any arm. 
For $m \in \mathbb{N}$ and $l \in \mathbb{N},$ define $\gamma_{m,l} := LM \left(\dfrac{3\log m}{2l}\right)^{\frac{\alpha}{2}}.$ The r.v. $\widehat \mu_{k,l}$, defined by Algorithm \ref{alg:cpt-estimator}, is an estimate of $\mu_k$ that uses the $l=T_k(m-1)$ sample rewards of arm $k$ seen so far and $\gamma_{m,l}$ is the confidence width, which together with $\widehat \mu_{m,l}$ ensures that the true weight-distorted value $\mu_k$ lies within  $[\widehat \mu_{k,l} - \gamma_{m,l}, \widehat \mu_{k,l} + \gamma_{m,l}]$ with high probability, i.e., 
\[\mathbb{P} \left(\mu_k \le \widehat \mu_{k,l} + \gamma_{m,l} \right) \le \dfrac{2}{m^{3}} \textrm{ and } \mathbb{P}\left(\widehat \mu_k  \ge \mu_{k,l} - \gamma_{m,l} \right) \le \dfrac{2}{m^{3}},\,\, k=1,\ldots,K.\] 

\begin{algorithm}[h]  
	\caption{W-UCB}
	\label{alg:w-ucb}
	\begin{algorithmic}
		\State {\bfseries Input:} Weight function $w$.
		\State In the first $K$ rounds, play each arm once.
		\State In round $m\ge K$, play arm $I_m = \argmax\limits_{k \in \{1,\ldots,K\}} \left[ \widehat \mu_{k, T_k(m-1)} + \gamma_{m, T_k(m-1)} \right]$,\\
		where $T_k(m-1)$ is the number of plays of arm $k$ up to round~$m,$ $\widehat\mu_{k, T_k(m-1)}$ is given in Algorithm~\ref{alg:cpt-estimator}, and $\gamma_{m,l} = LM \left(\frac{3\log m}{2l}\right)^{\frac{\alpha}{2}}.$
		
	\end{algorithmic}
\end{algorithm}

Using the empirical estimates $\widehat \mu_{k,l}$ of weight-distorted rewards together with confidence width $\gamma_{m,l}$, we propose an UCB-type algorithm for minimizing the regret $R_n$ defined in \eqref{eq:regret-karmed}. 
The pseudocode of the Weighted-Upper Confidence Bound (W-UCB) algorithm is given in Algorithm~\ref{alg:w-ucb}. 
As in the case of regular UCB, the W-UCB algorithm pulls each arm once in the initialization phase, and 
in any round $m,$ after the initialization phase, chooses the arm with the highest UCB value.  
Now, we present a result that gives an upper bound on the expected regret of the W-UCB algorithm. 
\begin{theorem}[\textit{Regret upper bound}]
\label{thm:BasicHolderRegretGap}
Under (A1)-(A2), the expected cumulative regret $R_n$ of W-UCB is bounded as follows:
$$ \E \left[ R_n \right] \le \sum\limits_{\{k:\Delta_k>0\}} \dfrac{3(2LM)^{2/\alpha}\log n}{2\Delta_k^{2/\alpha - 1}} + MK\left(1 + \dfrac{2\pi^2}{3} \right).
$$ 
\end{theorem} 
\begin{proof}
See Appendix~\ref{sec:appendix-gapdependentregret} in the supplementary material. 
\end{proof}
The theorem above involves the gaps $\Delta_k$. We next present a gap-independent regret bound in the following result:
\begin{corollary}[\textit{Gap-independent regret}]
\label{cor:BasicHolderRegretNoGap}
  Under (A1)-(A2), the expected cumulative
  regret of W-UCB satisfies the following gap-independent
  bound: there exists a universal constant $c>0$\footnote{The constant $c$ is universal as it does not depend on the problem-dependent quantities such as $K$, $n$, and the underlying distributions.} such that for all $n$,
  $ \E \left[ R_n \right] \le M K^{\alpha/2} \left(\frac{3}{2}(2L)^{2/\alpha} \log n + c \right)^{\frac{\alpha}{2}} \; n^{\frac{2-\alpha}{2}}.$ 
\end{corollary}
\begin{proof}
See Appendix~\ref{sec:appendix-gapindependentregret} in the supplementary material.
\end{proof}

\begin{remark}
We can recover the $O(\sqrt{n})$ regret bound (or the same dependence on the gaps) as in regular UCB for the case when $\alpha=1$, i.e., Lipschitz weights. On the other hand, when $\alpha <1$, the regret bounds are weaker than $O(\sqrt{n})$. Note that the weight function recommended in~\cite{tversky1992advances} is H\"{o}lder-continuous with exponent $\alpha$ strictly less than $1$.
\end{remark}
\begin{remark}
	In \cite{cassel2018general}, the authors propose the `Strongly stable EDPM' condition for a risk measure, under which one can obtain regret bounds. In our setting, for the weight-distorted reward for non-negative r.v.s, this condition requires the existence of a norm $|| \cdot ||$, on the space of distribution functions, and $q \geq 1, b > 0$ such that, for all distributions $F,G$ of non-negative r.v.s,
	\begin{align}
	\left| \intinfinity w(1-F(z))dz - \intinfinity w(1-G(z)) dz  \right| \le  b \left( || F- G || + || F- G||^q \right).\label{eq:ss}
	\end{align}
	The above condition does not hold, as illustrated by the following example. Consider $F$ and $G$ to be the distribution functions of Bernoulli distributions with parameters $4\epsilon$ and $\epsilon$, respectively, for $\epsilon > 0$, and take the weight distortion function to be $w(x) = \sqrt{x}$. The left hand side of \eqref{eq:ss} is $\sqrt{4\epsilon} - \sqrt{\epsilon} = \sqrt{\epsilon}$. On the other hand,  $(F - G)(x)  = \epsilon \mathbf{1}_{[0,1)}(x)$. By homogeneity of norms, $||F - G||$ must scale linearly with $\epsilon$, and hence the inequality in \eqref{eq:ss} cannot hold for sufficiently small $\epsilon$. 
	Thus, the approach of \cite{cassel2018general} cannot be adopted for handling a cumulative prospect theory-based risk measure. Our approach is to work directly with a concentration result for estimating distorted value, and obtain regret bounds. 
\end{remark}
\begin{remark}
For ease of exposition, we have focused on the case of bounded support in the bandit algorithms. However, it is straightforward to extend these algorithms to cover the sub-Gaussian case, and the modification here would be to pull each arm a certain number of times in the initialization phase, so that the concentration result for sub-Gaussian case given in Proposition~\ref{prop:sub-gaussian-concentration} applies. 
\end{remark}
The following result shows that one can not hope to obtain better regret than that of W-UCB (Theorem \ref{thm:BasicHolderRegretGap}) over the class of \holderNS-continuous weight functions, i.e., weight functions satisfying (A1)-(A2), by exhibiting a matching lower bound for regret, in the style of \cite{lai1985lowerbd}.


\begin{theorem}[\textit{Regret lower bound}]
\label{thm:karmedlowerbd}
Consider a learning algorithm for the $K$-armed weight-distorted bandit problem with the following property: For any weight distortion function $w:[0,1] \to [0,1]$ such that $w(0) = 0$ and $w(1) = 1,$ any set of reward distributions with rewards bounded by $M$, any $a > 0$ and any sub-optimal arm $k \in \{ 2, 3, \dots, K \},$ the expected number of plays of arm $k$ satisfies $\expect{T_k(n)} = o(n^a)$. 
Then, for any constant $\alpha \in (0,1]$ and $L \in (0, 2^{ \alpha - 1 } ],$ there exists a monotonically increasing weight function $w: [0, 1] \rightarrow [0, 1]$ with $w(0) = 0$ and $w(1) =1,$ which is $\alpha$-H\"{o}lder-continuous with constant $L$, and a set of reward distributions bounded by $M$, for which the algorithm's regret satisfies
\[  \liminf_{n \to \infty} \frac{\expect{R_n}}{\log n} \geq  \sum\limits_{k=2}^K \dfrac{(LM)^{2/\alpha}}{4\Delta_k^{2/\alpha - 1}}. \]
\end{theorem}

\begin{proof}[Proof sketch for Theorem~\ref{thm:karmedlowerbd}]
For convenience, a proof sketch is given for $M = 1$. The key ingredient in the proof is the seminal lower-bound on the number of suboptimal arm plays derived by Lai and Robbins \cite{lai1985lowerbd}. Their result shows that for any algorithm that plays suboptimal arms only a sub-polynomial number of times in the time horizon, 
\[ \liminf_{n \to \infty} \frac{\expect{T_k(n)}}{\log n} \geq \frac{1}{D(F_k || F_{1})} \]
for any suboptimal arm $k$ and we assumed that arm~1 is the optimal arm. It is worth noting that the argument at its core uses only a change-of-measure idea and the sub-polynomial regret hypothesis, and is thus unaffected by the fact that we measure regret by distorted, i.e., non-expected, rewards. Using this along with the definition of regret, we show that
\begin{align}
\liminf_{n \to \infty} \frac{\expect{R_n}}{\log n} &\geq \sum\limits_{k=2}^K \frac{\Delta_k}{D(F_k || F_{1})}. \label{eq:liminfbound_sketch}
\end{align}
We now construct a set of reward distributions and a weight function for which the limiting property claimed in the theorem holds. Consider Bernoulli distributions for the arms' rewards, i.e., arm $k$'s reward is Ber($p_k$). It gives a simple expression for the weight-distorted reward of any arm $k$ as $\mu_k = \int_0^1 w(p_k) dz = w(p_k)$, and, consequently, $\Delta_k = w(p_{1}) - w(p_k).$
Consider now a weight function $w:[0,1] \to [0,1]$ which is monotone increasing from $0$ to $1$, H\"{o}lder-continuous with any desired constant $L > 0$ and exponent $\alpha \in (0,1]$, and, moreover, satisfies the following ``\holder continuity property from below'': for some $\tilde{p} \in (0,1)$ (say, $\tilde{p} = 1/2$), $|w(\tilde{p}) - w(p)| \geq L |\tilde{p} - p|^\alpha$. Such a weight function can always be constructed, e.g., for $\alpha = 1/2, L = 1/\sqrt{2},$ take the function $w(x) = \frac{1}{2} - \frac{1}{\sqrt{2}}\sqrt{\frac{1}{2} - x}$ for $x \in [0, 1/2]$, and $w(x) = \frac{1}{2} + \frac{1}{\sqrt{2}}\sqrt{x - \frac{1}{2}}$ for $x \in (1/2, 1]$. (This is essentially formed by gluing together two inverted and scaled copies of the function $\sqrt{x}$, to make an S-shaped function infinitely steep at $x = 1/2$.)

For such a weight function, putting $p_{1} = \tilde{p} = 1/2$ and $p_k < p_{1}$, $k \neq 1$, we show that 
\begin{align*} 
D(F_k || F_{1}) 
&\leq 4 \left( \frac{\Delta_k}{L} \right)^{2/\alpha}. 
\end{align*}
The main claim can be easily established using the above inequality and~\eqref{eq:liminfbound_sketch}. 
For a detailed proof, the reader is referred to Appendix~\ref{sec:appendix-regretlowerbound} in the supplementary material.
\end{proof}

We now present a result which gives a problem-independent lower bound on the regret. 
\begin{corollary} 
\label{cor:prob-indepdent-regret-lowerbound}
Consider a H\"{o}lder-continuous weight function with parameters $\alpha$ and $L$. For any bandit algorithm $\mathcal{A}$, there exists a $K$-armed bandit problem instance $v$ such that 
\begin{align*}
R^{\mathcal{A}}_n(v) \geq c L n^{1-\frac{\alpha}{2}} K^{\frac{\alpha}{2}},
\end{align*}  
where $R^{\mathcal{A}}_n(v)$ is the regret incurred by the Algorithm $\mathcal{A}$ on the bandit problem instance $v$ until round $n$, and $c$ is a universal constant.
\end{corollary}
\begin{proof}
Follows directly from Theorem~1.12 in~\cite{prashanth6046}. 
\end{proof}

\subsection{Best arm identification with weight-based distortions}
\label{sec:algos-karmed-bestarm}
We now propose and study algorithms for the best arm identification problem for $K$-armed bandits under the fixed confidence and fixed budget settings. 
\subsubsection{Best arm identification with Fixed Confidence}  
We introduce the Weighted-Lower Upper Confidence Bound (W-LUCB) algorithm, which is a non-trivial adaptation of the well-known LUCB algorithm~\cite{kalyanakrishnan2012pac}, for best expected-value arm identification. Unlike the latter algorithm, W-LUCB incorporates weight-distorted reward estimates and confidence widths for arms to find the arm with the highest weight-distorted reward. 

Algorithm~\ref{alg:cpt-lucb} presents the pseudocode of W-LUCB.  A high-level description of the W-LUCB algorithm is as follows: In any round $m$, we calculate a Lower Confidence Bound (LCB) and Upper Confidence Bound (UCB) on the weight-distorted rewards of each arm, obtained by adding and subtracting a confidence term $\left( \gamma'_{(\cdot), (\cdot)} \right)$ to the weight-distorted reward estimate $\left( \hat{\mu}_{(\cdot), (\cdot)} \right)$ of the arm. Then, we find the arm, denoted $h_m$ in Algorithm \ref{alg:cpt-lucb}, with the highest weight-distorted reward estimate and its LCB.  Next, we find the arm, say $l_m$, which has the highest UCB among all arms excluding $h_m$. We sample from the distributions of $h_m$ and $l_m$, unless $LCB(h_m) > UCB(l_m)$, when the algorithm terminates to return $h_m$. This is justified because the weight-distorted reward of $h_m$ lies within its confidence interval with high probability and none of the other arms can possibly have a higher weight-distorted reward, since their confidence intervals do not intersect with that of $h_m.$ 
\begin{algorithm}[!h]  
\caption{W-LUCB}
\label{alg:cpt-lucb}
\begin{algorithmic}
\State {\bfseries Input:} $\delta \in (0, 1)$,  weight function $w$
\State In the first $K$ rounds, play each arm once.
\State Let $T_k(m-1)$ denote the number of times the algorithm has chosen arm $k$ up to round~$m,$ $\widehat\mu_{k, T_k(m-1)}$ is the weight-distorted reward estimate given by Algorithm~\ref{alg:cpt-estimator}.
\State Let ETop$_m = h_m = \arg \max\limits_{k \in \mathcal{A}} \widehat \mu_{k, T_k(m-1)}$, EBot$_m = \mathcal{A} \setminus$ ETop$_m$,
\State $\gamma'_{m,T_k(m-1)} = LM \left[ \frac{1}{2T_k(m-1)} \log \left( \frac{m^4 K \left[\pi^2/3 + 1\right]}{\delta} \right) \right]^{\frac{\alpha}{2}}$,
and  $l_m = \argmax\limits_{k \in EBot_m} \widehat \mu_{k, T_k(m-1)} + \gamma'_{m, T_k(m-1)}$.  
\While{$\left( \widehat \mu _{l_m, T_{l_m}(m-1)} + \gamma'_{m, T_{l_m}(m-1)} > \widehat \mu_{h_m, T_{h_m}(m-1)} - \gamma'_{m, T_{h_m}(m-1)} \right)$} 
\State Sample arms $h_m$ and $l_m$
\EndWhile 
\State {\bfseries Output:} Return $h_m$.
\end{algorithmic}
\end{algorithm}
The main results concerning the W-LUCB algorithm are given below.
\begin{theorem}[\textit{Correctness of W-LUCB}]
\label{Thm:correctness}
Assume (A1)-(A2) and let $\delta \in (0, 1)$. The following holds for the W-LUCB algorithm: 
\begin{align*}
\mathbb{P} \left( \text{W-LUCB does not return optimal arm} \right) < \delta.
\end{align*}
\end{theorem}
\begin{proof}
See Appendix~\ref{sec:proof-of-W-LUCB-correctness} in the supplementary material.
\end{proof}
\begin{theorem}[\textit{Sample complexity bound for W-LUCB}]
\label{Thm:LUCB6}
Assume $(A1)-(A2),$ for a given $\delta \in (0,1)$, let $N^{CL}$ denote the number of samples used by W-LUCB before termination. Then, we have
\begin{equation*}
\mathbb{E} \left[ N^{CL} \right] = O \left( H_\alpha \log \frac{H_\alpha}{\delta} \right),\text{where} \,\,\, H_\alpha~\footnote{For the sake of notational convenience, we have suppressed the dependence of $H_\alpha$ on $L$ and $M$.} = \sum\limits_{i \in \mathcal{A}} \left[ \frac{2LM}{\Delta_i} \right]^\frac{2}{\alpha}.
\end{equation*}  
\end{theorem} 
\begin{proof}
The proof follows the technique from~\cite{kalyanakrishnan2012pac}, and the reader is referred to Appendix~\ref{sec:proof-W-LUCB-complexity-bound} for the details.
\end{proof}
Theorems \ref{Thm:correctness}--\ref{Thm:LUCB6} together imply that, for any given $\delta \in (0, 1),$ W-LUCB terminates after using $N^{CL}$ samples (i.e., after $N^{CL}/2$ rounds) and returns the optimal arm w.p. at least $(1-\delta).$ In the next result, we show that any algorithm requires $\Omega \left( H_\alpha \right)$ samples to find the best arm under the fixed confidence setting. Hence, we call $H_\alpha$ as the \emph{hardness measure} of the problem of finding the best arm with the fixed confidence.

The key principle behind the algorithm and its performance guarantees (Theorems \ref{Thm:correctness} and \ref{Thm:LUCB6}) is as follows: In any round, we design the LCB and UCB so that at any time, the probability that the weight-distorted reward of an arm lies below (resp. above) its UCB (resp. LCB) is overwhelming high. Note that, standard concentration of measure bounds for the sample mean as an estimator of the expected value (such as Hoeffding's inequality) do not help to control estimates of the more complex weight-distorted reward functional. Hence, we leverage a concentration inequality for the weight-distorted reward estimate of Algorithm \ref{alg:cpt-estimator} \cite{aditya2016weighted}, relying on the Dvoretzky-Kiefer-Wolfowitz (DKW) inequality for concentration of the empirical CDF around the true CDF \cite{massart1990tight, wasserman2006}.

\begin{remark}
The classic expected-value optimizing algorithm LUCB requires $\tilde{O} \left( \sum\limits_{i \in \mathcal{A}} 1/ \Delta_i^2 \right)$ number of samples\footnote{$\tilde O(\cdot)$ is a variant of $O(\cdot)$ that ignores log-factors.}. In comparison, W-LUCB requires $\tilde{O} \left( \sum\limits_{i \in \mathcal{A}} \left( L/ \Delta_i\right)^{2/\alpha} \right)$ number of samples to find the arm with the highest weight-distorted reward. If $\alpha=1$, then the complexities match. However, a typical weight function used in the weight-distorted reward (see Figure \ref{fig:weight}) has $\alpha < 1$, leading to increased expected sample complexity for W-LUCB. 
\end{remark}

A natural question to ask is whether the upper bound on the expected sample complexity of W-LUCB algorithm can be improved. We now present an algorithm-independent lower bound on the expected sample complexity that answers this question in the negative. 
\begin{theorem}[\textit{Lower bound on sample complexity}]
\label{Thm:FCSBLB}
Let $\delta \in (0, 1), \alpha \in (0,1)$ and $L \in (0, 2^{\alpha - 1} ].$ Let $Alg(\delta, K)$ denote the class of algorithms which output the optimal arm w.p. at least $(1-\delta)$ on any $K$-armed bandit problem. Then, there exists a monotonically increasing weight function that satisfies $(A1)$ with parameters $\alpha, L,$ and a set of arm distributions bounded by $M$, such that any algorithm $A \in Alg(\delta, K)$ satisfies 

\begin{equation*}
\mathbb{E} \left[ N^A \right] \geq O \left( H_\alpha \log \left( \frac{1}{2.4 \delta} \right) \right),
\end{equation*} 
where $N^A$ is the rounded number at which algorithm $A$ terminates and $H_\alpha$ is as defined in Theorem \ref{Thm:LUCB6}. 
\end{theorem}
\begin{proof}
The proof uses an information theoretic lower bound given in~\cite{kaufmann2015complexity}, and the reader is referred to Appendix~\ref{sec:proof-complexity-lower-bound} for the details.
\end{proof}
Theorems~\ref{Thm:LUCB6} and~\ref{Thm:FCSBLB} imply that W-LUCB is optimal in expected sample complexity up to a factor of $\log \left( H_\alpha \right)$.
\subsubsection{Best arm identification with Fixed Budget}
\label{Sec:SBFB}
\begin{algorithm}[t]  
\caption{W-SR}
\label{alg:cpt-sr}
\begin{algorithmic}
\State {\bfseries Input:}  Budget $n$ of plays,  weight function $w$ 
\State {\bfseries Initialization:} Let $\mathcal{A}_1 = \mathcal{A},$ $\overline{\log} K = \frac{1}{2} + \sum\limits_{i=2}^K \frac{1}{i}$ and $n_0 = 0.$ 
\For{each phase $k = 1, 2, \cdots K-1 $}
\State Play each arm in $\mathcal{A}_k$ for $(n_k - n_{k-1})$ times, where  $n_k = \Big\lceil \frac{1}{\overline{\log} K} \frac{n-K}{K+1-k} \Big\rceil$.
\State Estimate weight-distorted rewards $\left( \widehat \mu_{i, (\cdot)} : i \in \mathcal{A}_k \right)$ for all arms in $\mathcal{A}_k$ using Algorithm~\ref{alg:cpt-estimator}
\State Set $\mathcal{A}_{k+1} = \mathcal{A}_k \setminus I_k$, where $I_k = \arg\min\limits_{i \in \mathcal{A}_k} \hat{\mu}_{i, (\cdot)}$.
\EndFor
\State {\bfseries Output:} Return $J_n = \mathcal{A}_K.$
\end{algorithmic}
\end{algorithm}
We begin by introducing the Weighted-Successive Rejects (W-SR) algorithm for the best arm identification problem with a fixed budget (Algorithm~\ref{alg:cpt-sr}), inspired by the Successive Rejects (SR) algorithm~\cite{audibert2010best}.  A high-level idea of the W-SR algorithm is as follows: We divide the given fixed budget into $K-1$ phases. In each phase, we play all available arms an equal number of times and eliminate the arm with the lowest weight-distorted reward estimate. Note that, the phase lengths are designed such that the optimal arm survives across all $(K-1)$ phases with overwhelming probability. Despite the phase lengths in both algorithms being the same, an arm that will get eliminated at the end of a phase in the W-SR algorithm depends on weight-distorted reward estimates of arms rather than empirical means of arms.

The following result provides an  upper bound on the probability of incorrect identification by W-SR. 
\begin{theorem}[\textit{Probability of incorrect identification by W-SR}]
\label{Thm:SR1}
Let $(A1)-(A2)$ hold. For a given budget $n$, the arm $J_n$ returned by the W-SR algorithm satisfies:
\begin{equation*}
\mathbb{P} \left( J_n \neq 1 \right) \leq \frac{K(K-1)}{2} \exp \left( - \frac{n-K}{\overline{\log} K} \frac{1}{M^{2/\alpha}} \frac{1}{H_{2, \alpha}} \right).
\end{equation*}
In the above, $\overline{\log} K = \frac{1}{2} + \sum\limits_{i=2}^K \frac{1}{i}$, $H_{2, \alpha} = \max\limits_{i\neq 1} i\left( \frac{L}{\Delta_i} \right)^{2/\alpha}$, where $\alpha$ is the \holder exponent, $L$ is the \holder constant and $M$ is a bound on the stochastic rewards from any arm's distribution (see (A1) and (A2)).
\end{theorem}
\begin{proof}
The proof uses the proof technique from~\cite{audibert2010best}, the reader is referred to Appendix~\ref{sec:proof-W-SR-error-upper-bound} for the details.
\end{proof}
\begin{remark}
From the result in the theorem above, it is apparent that $O(H_{2, \alpha})$ number of samples would be necessary for the SR algorithm to correctly identify the best arm, with high probability. 
Further, the dependence on the underlying gaps in $H_{2, \alpha}$ is $\tilde{O}\left(1/\Delta_2^{2/\alpha}\right)$, while the corresponding complexity term in classic SR is $\tilde{O}\left(1/\Delta_2^2\right)$.  
\end{remark}

In order to show that the dependence of W-SR on the underlying gaps is optimal, we proceed to present two lower bounds on the probability of incorrect identification. To that end, we construct $K$ transformations of a $K$-armed bandit, and show that any algorithm must return a sub-optimal arm with a sufficiently high probability on at least one of these transformations.  

We consider a $K$-armed Bernoulli bandit problem with $p_1 = 1/2, p_j \in [1/4, 1/2)$ for $2 \leq j \leq K$. Following the measure change technique from~\cite{carpentier2016tight}, we consider $K$ transformations of a $K$-armed bandit problem. 
For $1 \leq i \leq K,$ the transformation-$i$ corresponds to a $K$-armed bandit with Bernoulli arm distributions with parameters $p_1, p_2, \cdots, p_{i-1}, 1 - p_i, p_{i+1}, \cdots p_K$. We denote transformation-$i$ as $MAB_i.$ It is easy to see that for a H\"{o}lder-continuous and monotonically increasing weight function $w(\cdot)$, the weight-distorted reward of a Bernoulli random variable with parameter $p$ is $w(p)$. Hence, arm $i$ is the optimal arm under $MAB_i.$ Let $\mathbb{P}_i$ and $\mathbb{E}_i$ denote the joint probability distribution and expectation under $MAB_{i},$ respectively. 

Let $\lbrace \Delta^i_k \rbrace_k$ denote the gaps between the optimal arm and arm $k$ under $MAB_i,$ which are defined as follows: 
\begin{align*}
&\Delta^i_k := 
\begin{cases}
1-p_i - p_k, \hspace{0.85cm} \text{for $i \neq k,$}\\
1/2 - p_2,  \hspace{1.25cm} \text{for $i=k=1,$} \\
1/2 - p_i, \hspace{1.25cm} \text{for $2 \leq i=k \leq K.$}
\end{cases}
\end{align*}
Let $H(i,\alpha) = \sum\limits_{1 \leq k \leq K, k\neq i} \left( \frac{L}{\Delta^i_k} \right)^{2/\alpha}$ and $H_{1,\alpha} = \max\limits_i H(i, \alpha).$ We call $H_{1, \alpha}$ the \emph{hardness measure} of the problem of finding the best arm under the fixed budget setting. 
Using arguments similar to that in Section 6.1 of \cite{audibert2010best}, it can be  verified that 
\begin{align}
H_{2, \alpha} \le H_{1, \alpha} \le 2 \log(K) H_{2, \alpha}.\label{eq:hrels}
\end{align}
In the following, we provide two lower bounds on the probability of returning a sub-optimal arm by any algorithm on one of the $K$ transformations described above.
\begin{theorem}
\label{Thm:FBSBLB1}
Let $n$ be the given budget, $\alpha~\in~(0,~1)$ and $L \in (0, 2^{\alpha - 1} ].$ Then, there exists a monotonically increasing weight function that satisfies $(A1)$ with parameters $\alpha, L,$ and arm distributions bounded by 1 such that any algorithm satisfies the following:
\begin{eqnarray*}
\max\limits_{1 \leq i \leq K} \mathbb{P}_i (J_n \neq i) \geq \frac{1}{6} \exp \left[ \frac{-60 n}{H_{1,\alpha}} - 2 \sqrt{n \log(6nK)} \right]. 
\end{eqnarray*}
\end{theorem}
\begin{proof}
The proof employs the technique from~\cite{carpentier2016tight}, and the reader is referred to Appendix~\ref{sec:proof-k-armed-bai-fb-lowerbound} for the details.
\end{proof}

In the next result, we present an improved lower bound on the probability of incorrect identification, in the case of knowledge of the problem complexity is unknown an algorithm,  that matches with the upper bound of W-SR algorithm.

\begin{theorem}[\textit{Lower bound on probability of incorrect identification}]
\label{thm:bai_karmed_improved_lb}
Let $K > 1$, $a > 0 $, $\alpha \in (0, 1]$ and $L \in (0, 2^{\alpha - 1} ].$ Let $\mathbb{B}_a$ be the set of MAB problems with support of arm distributions lying in $[0, 1]$ and the problem complexity $H_\alpha$ bounded by $a$. For any MAB problem $G \in \mathbb{B}_a$, let $a^*(G)$ be the set of optimal arms and $H_\alpha(G)$ be the problem complexity of the corresponding problem $G.$ Let $\mathbb{P}_G(\cdot)$ denote the probability measure under the problem $G.$
\begin{itemize}
\item[(a)] If $n \ge a^2 (4 \log(6nK))/3600$, for any MAB algorithm that outputs an arm $J_n$ at time $n$, the following holds:
\begin{align}
\label{eq:bai_new_thm1}
\sup_{G \in \mathbb{B}_a} \mathbb{P}_G \left( J_n \notin a^*(G) \right) \geq \frac{1}{6} \exp\left( -120 \frac{n}{a} \right). 
\end{align}
\item[(b)] Furthermore, if $a \geq 2 (4K)^{2/\alpha}$ and $K \geq 2$, then for any MAB algorithm that outputs an arm $J_n$ at time $n$, the following holds:
\begin{align}
\label{eq:bai_new_thm2}
\sup_{G \in \mathbb{B}_a} \left[ \mathbb{P}_G \left( J_n \notin a^*(G) \right) \times \exp \left(  \frac{480 (1-2/\alpha)n}{\left[ (\log K)^{1-2/\alpha} - 2^{1-2/\alpha}  \right] H_\alpha(G)} \right) \right] \geq \frac{1}{6}. 
\end{align}
\end{itemize}
\end{theorem}
\begin{proof}
The proof uses the result in Theorem~\ref{Thm:FBSBLB1} -- See Appendix~\ref{proof:bai_karmed_improved_lb} for the details.
\end{proof}

The lower and upper bounds together with \eqref{eq:hrels} suggest that W-SR is an order-optimal algorithm.

\section{Linear bandit models}
\label{sec:model-linear}
In this section, we formally present the models for linear bandits under regret minimization and best arm identification settings. We first introduce the required notation. 

Let $\X \subseteq \mathbb{R}^d$ be the set of arms available to the learner. Let $\theta_* \in \mathbb{R}^d$ be the underlying parameter unknown to the learner. In round-$m,$ if the learner chooses arm $x_m$ from $\X,$ then the learner receives a stochastic reward given by
\begin{align}
r_m := x_m\tr\left(\theta_* + N_m\right),\label{eq:linban-depnoise}
\end{align} 
where $N_m := (N_m^1, \ldots, N_m^d)$ is a vector of i.i.d. standard Gaussian random variables\footnote{Our results continue to hold when the noise distribution is sub-Gaussian --- see Remark~\ref{remark:sub-gauss-linear}.}, independent of the previous vectors $N_1, \ldots, N_{m-1}.$ We assume that the learner does not have the knowledge of $N_m.$ Given a weight function $w: [0,1] \to [0,1]$ with $w(0) = 0$ and $w(1) = 1,$ we define the weight-distorted reward $\mu_x(\theta_*)$ for arm $x \in \X$, with underlying model parameter $\theta_*$, to be the quantity
\begin{align}
\mu_x(\theta_*)  := \intinfinity w(1 - F_x^{\theta_*}(z)) dz - \intinfinity w(F_x^{\theta_*}(-z)) dz, \label{eq:cpt-lb} 
\end{align}
where $F_x^{\theta_*}(z) := \prob{x\tr(\theta_* + N) \leq z}$, $z \in \R$, is the cumulative distribution function of the stochastic reward from playing arm $x \in \X$. An arm $x$ is said to be optimal if its weight-distorted reward equals the highest possible weight-distorted reward achieved across all arms, i.e., if
$ \mu_x = \mu_* :=  \max_{x' \in \X} \mu_{x'}(\theta_*),$ and we use $x_*$ to denote the optimal arm. 
\begin{remark}
For purely notational convenience and ease of readability, we have used the same weight function for both positive and negative rewards in~\eqref{eq:cpt-lb}, but our results still hold for the setup where two different H\"{o}lder-continuous weight functions are considered for positive and negative rewards. 
\end{remark}

\subsection{Regret minimization with weight-based distortions}
\label{sec:model-linear-regret}
The setting here involves arms that are given as the compact set $\X\subset\R^d$ (each element of $\X$ is interpreted as a vector of features associated with an arm). 
The learning game proceeds as follows. At each round $m = 1, 2, \ldots, n$, the learner 
\begin{itemize}
\item[(a)] plays an arm $x_m \in \X$, possibly depending on the history of observations thus far, and 
\item[(b)] observes a stochastic, non-negative reward given by $r_m := x_m\tr\left(\theta_* + N_m\right).$ 
\end{itemize}

As in the $K$-armed setting, the performance measure is the cumulative regret $R_n$ over $n$ rounds, defined as $R_n =  \sum_{m=1}^n \mu_{x_{m}}(\theta_*) - n\mu^*,$ where $x_{m}$ is the arm chosen by the bandit algorithm in round~$m$. Our goal is to propose algorithms that have small cumulative regrets. 

\subsection{Best arm identification with weight-based distortions}
\label{sec:model-linear-bestarm}
Here, we assume that the arms are given by a set $\X$ with $\vert \X \vert = K$ (typically $K \gg 1).$ In each round-$m$, the learner chooses an arm $x_m$ and receives a stochastic reward $r_m.$ Similar to the $K$-armed bandit setting, here, the learner's goal is to identify the optimal or best arm. We study this problem with fixed confidence setting, \emph{i.e.,}  for a given $ \delta \in (0, 1),$ the learner's goal is to find the optimal arm w.p. at least $(1-\delta)$, while minimizing the number of sample observations. As mentioned in Section~\ref{sec:model-karmed-bestarm}, here too an algorithm's performance is quantified by its sample complexity, which is nothing but the number of samples it uses before it stops.

\section{Linear bandit algorithms}
\label{sec:algos-linear}
In this section, we present algorithms and their performance guarantees for linear bandits under the regret minimization and best arm identification settings. 

\subsection{Regret minimization with weight-based distortions}
\label{sec:algos-linear-regret}
In Algorithm \ref{alg:WOFUL}, we present a pseudocode of the W-OFUL algorithm for solving the problem described in Section \ref{sec:model-linear-regret}. Algorithm \ref{alg:WOFUL} follows the general template for linear bandit algorithms (cf. ConfidenceBall in \cite{dani2008stochastic} or OFUL in \cite{abbasi2011improved}), but deviates in the step in which an arm is chosen. Let $C_m$ be the confidence ellipsoid that is specified in Algorithm \ref{alg:WOFUL}, and $\mu_x(\theta)$ be the weight-distorted value defined in \eqref{eq:cpt-lb} for any arm $x \in \X$ and $\theta \in C_m$. It is worth noting that, in any round $m$ of the algorithm, the W-OFUL uses $\mu_x(\theta)$ as the decision criterion, whereas regular linear bandit algorithms use $x\tr\theta$.  Note that the ``in-parameter'' or arm-dependent noise model \eqref{eq:linban-depnoise} also necessitates modifying the standard confidence ellipsoid construction of \cite{abbasi2011improved} by rescaling with the arm size (the $A_m$ and $b_m$ variables in Algorithm \ref{alg:WOFUL}). For a positive semidefinite matrix $P$ and a vector $x$, we use the notation $\norm{x}_P = \sqrt{x\tr P x}$ to denote the Euclidean norm of $x$ weighted by $P$.  

\algblock{UCBcompute}{EndUCBcompute}
\algnewcommand\algorithmicUCBcompute{\textbf{\em Confidence set computation}}
\algnewcommand\algorithmicendUCBcompute{}
\algrenewtext{UCBcompute}[1]{\algorithmicUCBcompute\ #1}
\algrenewtext{EndUCBcompute}{\algorithmicendUCBcompute}

\algblock{RecoAndFeedback}{EndRecoAndFeedback}
\algnewcommand\algorithmicRecoAndFeedback{\textbf{\em Arm selection + feedback}}
\algnewcommand\algorithmicendRecoAndFeedback{}
\algrenewtext{RecoAndFeedback}[1]{\algorithmicRecoAndFeedback\ #1}
\algrenewtext{EndRecoAndFeedback}{\algorithmicendRecoAndFeedback}

\algblock{StatsUpdate}{EndStatsUpdate}
\algnewcommand\algorithmicStatsUpdate{\textbf{\em Update statistics}}
\algnewcommand\algorithmicendStatsUpdate{}
\algrenewtext{StatsUpdate}[1]{\algorithmicStatsUpdate\ #1}
\algrenewtext{EndStatsUpdate}{\algorithmicendStatsUpdate}

\algtext*{EndUCBcompute}
\algtext*{EndRecoAndFeedback}
\algtext*{EndStatsUpdate}

\begin{algorithm} 
	\caption{W-OFUL}
	\label{alg:WOFUL}
	\begin{algorithmic}
		\State {\bfseries Input:} regularization constant $\lambda \geq 0$, confidence $\delta \in (0,1)$, bound $\beta$ such that $\norm{\theta_*}_2 \leq \beta$, and the weight function~$w$.
		\State {\bfseries Initialization:} $A_1=\lambda I_{d \times d}$ ($d \times d$ identity matrix), $b_1=0$, $\hat\theta_1 =0$.
		\For{$m = 1,2,\ldots n$}
		\State Set $C_{m} := \left\{\theta \in \mathbb{R}^d: \norm{\theta - \hat{\theta}_m}_{A_m} \leq D_m \right\}$ and  $D_m := \sqrt{2 \log \left(\frac{\det(A_m)^{1/2} \; \lambda^{d/2} }{\delta} \right)} + \beta\sqrt{\lambda}$.
		
		\State Let $ (x_{m}, \tilde \theta_{m}) := \argmax\limits_{(x,\theta') \in \X\times  C_{m}} \mu_x(\theta')$.
		
		\State Choose arm $ x_{m}$ and observe reward $r_m$.
		
		\State Update $A_{m+1} = A_m + \frac{x_{m}x_{m}\tr}{\norm{x_m}^2}$,
		$b_{m+1} = b_m+ \frac{r_m x_{m}}{\norm{x_m}}$, and
		$\hat\theta_{m+1} = A_{m+1}^{-1} b_{m+1}$.
		\EndFor
	\end{algorithmic}
\end{algorithm}

\subsubsection*{Computational complexity of W-OFUL}
W-OFUL can be implemented with the same computational complexity as OFUL when the number of arms is finite. In particular, the computationally intensive step in W-OFUL is the optimization of the weight-distorted value over an ellipsoid in the parameter space (the third line in the {\bf for} loop). This can be explicitly solved as follows. For a fixed $x \in \X$, we can let 
\[\bar \theta_{m,x} := \argmax\limits_{\theta' \in C_{m}} \mu_x(\theta') = \argmax\limits_{\theta' \in C_{m}} x\tr\theta' =  D_m A_m^{-1}x/\norm{x}_{A^{-1}} - \hat{\theta}_m.\] This is because the weight-distorted value is monotone under translation (see Lemma \ref{lemma:cptdiff} below). The reward-maximizing arm is thus computed as \[x_m = \argmax \{ \mu_{x_1}(\bar \theta_{m,1}), \ldots, \mu_{x_{|\X|}}(\bar \theta_{m, |\X|})\}.\]

\begin{remark}
	\label{remark:woful-form}
In a risk-neutral linear bandit setting, one is concerned with the expected value $x\tr \theta^*$ only, and an estimate of $\theta^*$ would be enough to estimate the expected value. In contrast, weight-distorted value estimation requires knowledge of the entire distribution. As a first step towards studying the weight-distorted measure in a linear bandit setting, we assume knowledge of the form of the underlying (sub-Gaussian) distribution, in particular, to compute $\mu_{(\cdot)}(\cdot)$ in the algorithm above. It should be possible to get rid of this assumption by estimating the underlying distribution in an online fashion, but deriving such an extension is beyond the scope of our work.
\end{remark}

The following result bounds the regret of the W-OFUL algorithm under mild conditions on the arms' rewards distributions and for any non-linear weight function $w$ that is bounded in $[0,1]$ (unlike the K-armed setting, we do not impose a \holder continuity assumption on $w$). 
\begin{theorem}[\textit{Regret bound for W-OFUL}]
\label{thm:linear-bandit-regret} 
Consider a weight function $w : [0, 1] \rightarrow [0, 1]$ with $w(0) = 0,$ $w(1) = 1$ and 
$\forall x \in \X: x\tr \theta \in [-1,1]$, and $\norm{\theta}_2 \leq \beta$. Then, for any $\delta > 0$, the regret $R_n$ of W-OFUL, run with parameters $\lambda > 0$, $B$, $\delta$ and $w$,
  satisfies\\[0.5ex]
\centerline{$\mathbb{P}\left(R_n \le \sqrt{32 d n D_n \log n} \; \; \forall n \geq 1 \right) \ge 1-\delta.$}
\end{theorem}
If, for all $x \in X$, $\norm{x}_2 \leq \ell$, then the quantity $D_n$ appearing in the regret bound above is $O\left( \sqrt{d \log \left( \frac{n\ell^2}{\lambda \delta}\right)}\right)$ \cite[Lemma 10]{abbasi2011improved}; thus, the overall regret is $\tilde{O}\left( d \sqrt{n} \right)$.


\begin{remark}
For the identity weight function $w(p) = p$, $0 \leq p \leq 1$ with $L = \alpha = 1$, we recover the stochastic linear bandit setting, and the associated $\tilde{O}\left( d \sqrt{n}\right)$ regret bound for linear bandit algorithms such as $\text{ConfidenceBall}_1$ and $\text{ConfidenceBall}_2$ \cite{dani2008stochastic}, OFUL \cite{abbasi2011improved}. Hence, the result above is a generalization of regret bounds for standard linear bandit optimization to the case where a non-linear
    weight function of the reward distribution is to be optimized from linearly parameterized observations. The distortion of the reward distribution via a weight function, rather interestingly, does \emph{not} impact the order of the bound on problem-independent regret, and we obtain $\tilde{O}\left( d \sqrt{n}\right)$ here, as well. 
\end{remark}

\begin{remark}
A lower bound of essentially the same order as Theorem \ref{thm:linear-bandit-regret} ($O\left(d \sqrt{n}\right)$) holds for regret in (undistorted) linear bandits \cite{DanKakHay07:LinBandit}. One can show a similar lower bound argument with distortions, implying that the result of the theorem is not improvable (order-wise) across weight functions.
\end{remark}

\noindent\textbf{\textit{Linear bandits with arm-independent additive noise}}: An alternative to modeling ``in-parameter'' or arm-dependent noise \eqref{eq:linban-depnoise} is to have independent additive noise, i.e., $r_m := x_m\tr \theta_* + \eta_m$. This is a standard model of stochastic observations adopted in the linear bandit literature \cite{abbasi2011improved,dani2008stochastic}. The key difference here is that, unlike the setting in \eqref{eq:linban-depnoise}, the noise component $\eta_m$ does \emph{not} depend on the arm played $x_m$. In this case, Lemma \ref{lemma:cptdiff} below shows that $\mu_{X+a} \geq \mu_X$, i.e., the weight-distorted reward $\mu$ preserves order under translations of random variables. As a consequence of this fact, the W-OFUL algorithm reduces to the OFUL algorithm in the standard linear bandit setting with arm-independent noise.  The reduction is in the sense that both WOFUL and regular OFUL algorithms choose the same sequence of arms at time instants $1,\ldots,n$, assuming the same noise values $\eta_1,\ldots,\eta_n$. This is because the arm $x_m$ chosen at instant $m$ is impacted by the ordering of the arms and both weight-distorted value $\mu_x(\theta')$ and $x\tr\theta'$ result in the same order for the arms due to the aforementioned fact regarding translation for $\mu$.

We provide a sketch of the proof of Theorem \ref{thm:linear-bandit-regret} below, while a detailed proof  is provided in Appendix~\ref{thm:woful_linear_regret_bound_proof} in the supplementary material.
\begin{proof}[Proof of Theorem \ref{thm:linear-bandit-regret} (Sketch)]
We upper-bound the instantaneous regret $ir_m$ as follows:
Letting $\hat x_m = \frac{x_m}{\norm{x_m}}$ and $\stdnormal$ to be a standard Gaussian r.v. in $d$ dimensions,  we have
\begin{align}
ir_m =  \mu_{x_*}(\theta_*) - \mu_{x_{m}}(\theta_*) \le \mu_{x_m}(\tilde\theta_m) - \mu_{x_{m}}(\theta_*) &= \norm{x_m} \left( \mu_{W + \hat x_{m}\tr\tilde\theta_m} - \mu_{W+ \hat x_{m}\tr\theta_*} \right) \label{eq:trm2}\\
& \le 2 \norm{x_m} \left| \hat x_{m}\tr(\tilde\theta_m - \theta_*) \right|, \label{eq:trm3}
\end{align} 
and the rest of the proof uses the standard confidence ellipsoid result that ensures $\theta_*$ resides in $C_m$ with high probability. 
 A crucial observation necessary to ensure \eqref{eq:trm2} is that, for any r.v. $X$ and any $a\in\R$, the difference in weight-distorted reward $\mu_{X+a} - \mu_X$ is a non-linear function of $a$, as shown in the Lemma \ref{lemma:cptdiff} below (See Appendix~\ref{sec:weight-distortion-undert-translation} for a proof). 
\begin{lemma}[weight-distorted reward under translation] 
\label{lemma:cptdiff}
Let\\ 
$\mu_X := \intinfinity w(\prob{X > z}) dz - \intinfinity w(\prob{-X > z}) dz.$ Then, for any $a \in \R$, we have 
\[\mu_{X+a} = \mu_X + \int_{-a}^0 \left[ w(\prob{X > u}) + w(\prob{X < u})  \right] du.\]
In addition, if the weight function $w$ is bounded above by $1$, then 
\begin{align*}
|\mu_{X+a} - \mu_X| \leq 2|a|, \text{ for any }a\in \R.
\end{align*}
\end{lemma}
Thus, it is not straightforward to compute the weight-distorted reward after translation, and this poses a significant challenge in the analysis of W-OFUL for the arm-dependent noise model that we consider here. 
\end{proof}

\subsection{Best arm identification with weight-based distortions}
\label{sec:algos-linear-bestarm}
In this section, we develop and study an algorithm, called W-G, for the best arm identification problem for linear bandits with the fixed confidence setting. The W-G algorithm is adapted from the G-allocation algorithm in~\cite{soare2014best} to incorporate weight-distorted reward criteria.
\subsubsection{Notation}
For any arm $x \in \X,$ denote by $S(x)$ the set of parameters $\theta \in \mathbb{R}^d$ for which arm~$x$ is the optimal arm, i.e., $S(x) = \lbrace \theta \in \mathbb{R}^d: \mu_x ( \theta) \geq \mu_{x'}(\theta)\,\, \forall x' \in \mathcal{X} \rbrace.$ Here $\mu_x(\theta)$, defined in \eqref{eq:cpt-lb}, is the weight-distorted reward of arm $x$ with underlying parameter $\theta$. 
Recall that, $\theta_*$ is the true unknown model parameter and $x_*$ is the optimal arm. For any $x \in \X,$ let $\Delta_x = \mu_{x_*}( \theta_*) - \mu_x( \theta_*),$ and $\Delta_{\min} = \min_{x \in \mathcal{X} \setminus x^*} \Delta_x.$ We assume that $\Delta_{\min} > 0.$
For a given sequence of arms ${\bf{x_n}} = \lbrace x_m \rbrace_{m=1}^n$ and $\lambda > 0,$ let $A_{\bf{x_n}} =\lambda I_{d\times d} + \sum\limits_{m=1}^n \frac{x_m x_m^T}{\Vert x_m \Vert^2}.$ For notational convenience, we write $A_{\bf{x_n}}$ as $A_n$ whenever the sequence of actions can be understood from the context. Let $\hat{\theta}_n$ be the ridge regression estimate of $\theta^*$ using $n$ observations, i.e., $\hat{\theta}_n = A_n^{-1} b_n,$ where $b_n = \sum\limits_{m=1}^n \frac{x_m r_m}{\Vert x_m \Vert}.$ For any $x, x' \in \X,$ let $\Delta(x, x') = \mu_x(\theta^*) - \mu_{x'}( \theta^*).$ 

\subsubsection{Stopping condition and arm selection in W-G}
Recall the confidence ellipsoid, $C_m,$ defined in Algorithm~\ref{alg:WOFUL}. Since $\theta_*$ lies within the confidence ellipsoid $C_n$ with high probability \cite[Theorem~2]{abbasi2011improved}, a reasonable condition to stop the algorithm is when the confidence ellipsoid is contained in $S(x)$, for some $x \in \mathcal{X}$. 
If $C_n \subseteq S(x)$ in some round $n$, then
\[\forall \theta \in C_n, \mu_x(\theta) \geq \mu_{x'}(\theta),\,\, \forall x' \in \mathcal{X},\] 
or equivalently, 
\begin{equation}
\label{eq:Stop1}
\left[ \mu_x(\hat{\theta}_n) - \mu_x(\theta) \right] - 
\left[ \mu_{x'}(\hat{\theta}_n) - \mu_{x'}(\theta) \right] \leq \hat{\Delta}_n(x, x'),
\end{equation}
where $\hat{\Delta}_n(x, x') = \mu_x( \hat{\theta}_n) - \mu_{x'}(\hat{\theta}_n)$.
Note that verifying~\eqref{eq:Stop1} is difficult in practice, and hence we derive a simpler sufficient condition as follows:
\begin{align}
\left[ \mu_x(\hat{\theta}_n) - \mu_x(\theta) \right] - \left[ \mu_{x'}(\hat{\theta}_n) - \mu_{x'}(\theta) \right] & \leq \left\vert x^T \left( \hat{\theta}_n - \theta \right) \right\vert + \left\vert x'^T \left( \hat{\theta}_n - \theta \right) \right\vert, \label{eq:lemma5app} \\
& \leq \Vert x \Vert_{A^{-1}_n} \Vert \hat{\theta}_n - \theta \Vert_{A_n} +\Vert x' \Vert_{A^{-1}_n} \Vert \hat{\theta}_n - \theta \Vert_{A_n}, \label{eq:a} \\
& = \left[ \Vert x \Vert_{A^{-1}_n} + \Vert x' \Vert_{A^{-1}_n} \right] \Vert \hat{\theta}_n - \theta \Vert_{A_n}, \label{eq:Stop2}
\end{align}
where \eqref{eq:lemma5app} follows from the fact that $\vert \mu_{Z+a} - \mu_Z \vert \leq 2 \vert a \vert$ for any r.v. $Z$ and $a \in \mathbb{R}$~(due to Lemma~\ref{lemma:cptdiff}), while \eqref{eq:a} is due to the Cauchy-Schwartz inequality. From~\eqref{eq:Stop2}, a sufficient condition for~\eqref{eq:Stop1} is given by 
\begin{equation*}
\forall \theta \in C_n, \exists x \in \mathcal{X} \,\, \text{s.t.} \,\, \left[ \Vert x \Vert_{A^{-1}_n} + \Vert x' \Vert_{A^{-1}_n} \right] \Vert \hat{\theta}_n - \theta \Vert_{A_n} \leq \hat{\Delta}_n(x, x'), \,\, \forall x' \in \mathcal{X}.
\end{equation*}
Using the fact that $\Vert \hat{\theta}_n - \theta \Vert_{A_n} \leq D_n$ for any $\theta \in C_n,$ we have that condition~\eqref{eq:Stop1} holds if
\begin{align}
\exists x \in \mathcal{X} \,\, \text{s.t.} \,\, \left[ \Vert x \Vert_{A^{-1}_n} + \Vert x' \Vert_{A^{-1}_n} \right] D_n \leq \hat{\Delta}_n(x, x'), \,\, \forall x' \in \mathcal{X}.\label{eq:some12} 
\end{align}
Using the fact that for any $x \in \mathcal{A},$ $\Vert x \Vert_{A^{-1}_n} \leq \max\limits_{x \in \mathcal{A}} \Vert x \Vert_{A^{-1}_n},$ it is easy to see that the following condition implies \eqref{eq:some12}: 
\begin{equation}
\label{eq:Stop100}
\exists x \in \mathcal{X} \, \text{ s.t. } \, 2\max\limits_{x \in \mathcal{X}} \Vert x \Vert_{A^{-1}_n} D_n \leq \hat{\Delta}_n(x, x'), \, \forall x' \in \mathcal{A}. 
\end{equation}
Thus, we have arrived at a condition that can be easily verified as compared to \eqref{eq:Stop1}, which is a computationally intensive operation as the latter involves checking an inequality for each candidate $\theta$ within the ellipsoid $C_n$. 

\begin{algorithm}[h]  
	\caption{W-G}
	\label{alg:cpt-g}
	\begin{algorithmic}
		\State {\bfseries Input:} $\delta \in (0, 1),$ $\beta,$ $\lambda,$ and weight function $w$
		\While{(Stopping condition~\eqref{eq:Stop100} does not hold)}
		\State Choose $x_m
		= \arg\min\limits_{x' \in \mathcal{X} }  \max\limits_{x \in \mathcal{X}} \Vert x \Vert_{\left( A_{\bf{x_{m-1}}} + x' x'^T \right)^{-1}},$ and observe reward $r_m.$ 
		\State Update $A_{\bf{x_m}} = A_{\bf{x_{m-1}}} + \frac{x_m x_m^T}{\Vert x_m \Vert^2},$ $b_m = b_{m-1} + \frac{x_m r_m}{\Vert x_m \Vert },$ $\hat{\theta}_m = A_{\bf{x_m}}^{-1} b_m$
		\EndWhile
		\State {\bfseries Output:} Return arm $x$ which satisfies~\eqref{eq:Stop100}. 
	\end{algorithmic}
\end{algorithm}

We now present an algorithm that finds an arm $x$ that satisfies \eqref{eq:Stop100}, while minimizing the number of the sample observations. Notice that, in order to meet the stopping condition, we need to choose a sequence of actions which minimizes the left hand side in \eqref{eq:Stop100}, i.e.,  
\begin{equation}
\label{eq:arm-selection-cpt-g-1}
{\bf{x_n^G}} = \arg\min\limits_{ {\bf{x_n}} }  \max\limits_{x \in \mathcal{X}} \Vert x \Vert_{A_{\bf{x_n}}^{-1}}.
\end{equation}
The arm selection strategy above is for the first $n$ rounds. However, $n$ is not known a priori.  So, we employ an incremental version of the arm selection strategy in \eqref{eq:arm-selection-cpt-g-1}, as shown in Algorithm~\ref{alg:cpt-g}, where we present the pseudocode for W-G. 
As in the case of W-OFUL, the W-G algorithm requires the knowledge of the form of the underlying distribution (see Remark \ref{remark:woful-form}), and W-G uses this knowledge to calculate the weight-distorted reward in the stopping condition.

We now present an upper bound on the sample complexity of W-G algorithm below.
\begin{theorem}[\textit{Sample complexity bound}]
\label{Thm:LinBdt1}
Consider a weight function $w : [0, 1] \rightarrow [0, 1]$ with $w(0) = 0$ and $w(1) = 1.$ For any $\delta \in (0, 1),$ let $N^G$ be the round number at which W-G algorithm stops and $\Pi \left( N^G \right)$ be the arm returned at $N^G.$ Then, we have  
\begin{equation}
\label{eq:CPT-Gbound}
\mathbb{P} \left( N^G \leq \frac{ c D_n^2 d}{\Delta_{\min}^2}, \,\, \Pi\left( N^G \right) = a^* \right)\geq 1 - \delta,
\end{equation}
where $d$ is the dimension of the space in which $\theta^*$ lies, and $c$ is a universal constant.
\end{theorem}
\begin{proof}[Proof sketch for Theorem~\ref{Thm:LinBdt1}]
First, using Lemma~\ref{lemma:cptdiff} and the Cauchy-Schwartz inequality we show that: 
\begin{align}
\hat{\Delta}_n \left( x^*, x \right) \geq \Delta \left( x^*, x \right) -2 \Vert \hat{\theta}_n - \theta^* \Vert_{A_n} \left[ \Vert x^{*T} \Vert_{A_n^{-1}} +  \Vert x^T \Vert_{A_n^{-1}} \right].
\end{align}
Then using the above, a high probability concentration bound on the confidence ellipsoids, Lemma~\ref{lemma:confidenceellipsoid} in the Appendix, and $\rho^G_n = n \max\limits_{x \in \mathcal{X}} x^T A^{-1}_n x$ we show that the following holds with probability at least $(1 - \delta)$ 
\begin{align}
\hat{\Delta}_n \left( x^*, x \right) \geq  \Delta_{\min} -2 D_n \sqrt{\frac{\rho^G_n}{n}}.
\end{align}
Using a sufficient condition on stopping rule given in equation~\eqref{eq:Stop100} and the above, we show that 
\begin{align}
\Delta_{\min}^2 \geq \frac{ 9 D_n^2 \rho^G_n}{n}.
\end{align}
Finally, the theorem is established by choosing a suitable value for the horizon $n.$
The reader is referred to Appendix~\ref{sec:proof-W-G-upperbound} for a detailed proof.
\end{proof}

\begin{remark}
The upper bound on $N^G$ in~\eqref{eq:CPT-Gbound} is independent of the parameters of the weight function and the number of arms $(K)$, but depends on the dimension of the underlying space~$(d).$ Also note that, in terms of dependence on the underlying gaps of arms, this upper bound is of the same order as the upper bound on the sample complexity of G-allocation in~\cite{soare2014best} for the problem of identifying the arm with the highest mean in linear bandits. The bounds differ only in the $\Delta_{min}$ term; here it is a function of weight-distorted rewards of arms, whereas in~\cite{soare2014best}, it is a function of expected values of arms. 
\end{remark}

\section{Vehicular traffic routing application}
\label{sec:experiments-regret}

We study the problem of optimizing the route choice of a human traveller using the Green light district (GLD) traffic simulation software \cite{GLDSim}. In this setup, a source-destination pair is fixed in a given road network. Learning proceeds in an online fashion, where the algorithm chooses a route in each round and the system provides the (stochastic) delay for the chosen route. The objective is to find the ``best'' path that optimizes some function of delay, while not exploring too much. While traditional algorithms minimize the expected delay, in this work, we consider the distorted value (as defined in \eqref{eq:cpt-lb}) as the performance metric. 

The gains/losses in CPT are usually w.r.t. a baseline \cite{tversky1992advances}.  Motivated by this observation, the distorted delay calculation in the experiments includes a baseline. In this setting, the goal is to minimize the excess delay, either in expectation or under the distortion measure. To elaborate, the r.v. considered is  $x\tr(\theta_{*} + N) -B$, where $B$ is the baseline delay, while the rest of the symbols are as defined in Section \ref{sec:model-linear}. 
For our experiments, we pick the average delay of a random routing algorithm as the baseline delay: a choice used previously in \cite{avineri12002sensitivity}. 

We implement both OFUL and W-OFUL algorithms for this problem. 
OFUL algorithm finds a route that has the lowest excess delay, while WOFUL finds a route with the lowest distorted excess delay. The latter quantity is calculated using \eqref{eq:cpt-lb} for the r.v. mentioned above. 
The distorted excess delay is calculated using the following weight function: $w(p) = \frac{p^{0.1}}{{(p^{0.1}+ (1-p)^{0.1})}^{10}}.$
Since the weight function $w$ is non-linear, a closed-form expression for $\mu_x(\hat\theta_m)$ is not available, and we employ the empirical distribution scheme, described for the $K$-armed bandit setting (see Algorithm \ref{alg:cpt-estimator}), for estimating the weight-distorted value. For this purpose, we simulate $25000$ samples of the Gaussian distribution, as defined in \eqref{eq:linban-depnoise}.

\begin{figure}
	\centering
	\scalebox{0.8}{
		\tikzset{roads/.style={line width=0.1cm}}
		
		\tabl{c}{\begin{tikzpicture}
				\filldraw (0,0) node[color=white,font=\bfseries]{1} circle (0.4cm);
				\filldraw (1.5,0) node[color=white,font=\bfseries]{2} circle (0.4cm);
				\filldraw (3,0) node[color=white,font=\bfseries]{3} circle (0.4cm);
				
				\filldraw (0,-6.0) node[color=white,font=\bfseries]{7} circle (0.4cm);
				\filldraw (1.5,-6.0) node[color=white,font=\bfseries]{8} circle (0.4cm);
				\filldraw (3,-6.0) node[color=white,font=\bfseries]{9} circle (0.4cm);
				
				\filldraw (-1.5,-0-1.5) node[color=white,font=\bfseries]{10} circle (0.4cm);
				\filldraw (-1.5,-1.5-1.5) node[color=white,font=\bfseries]{11} circle (0.4cm);
				\filldraw (-1.5,-3-1.5) node[color=white,font=\bfseries]{12} circle (0.4cm);
				
				\filldraw (4.5,-0-1.5) node[color=white,font=\bfseries]{4} circle (0.4cm);
				\filldraw (4.5,-1.5-1.5) node[color=white,font=\bfseries]{5} circle (0.4cm);
				\filldraw (4.5,-3-1.5) node[color=white,font=\bfseries]{6} circle (0.4cm);
				
				\foreach \x in {0,1.5,3}
				{
					\draw[roads,] (\x,-0.25) -- (\x,-5.75);
					\draw[roads] (-1.25,-\x-1.5) -- (4.25,-\x-1.5);
				}
				
				\draw[dotted,blue,line width=0.1cm] (-1.25,-3-1.35)  -- (1.4,-3-1.35) -- (1.4, -1.6)-- (4.5, -1.6);
				\draw[dotted,green,line width=0.1cm] (-1.15,-3-1.65)  -- (2.9,-3-1.65) -- (2.9,-1.5-1.62) -- (-0.1,-1.5-1.62) -- (-0.1,-1.4)--(4.12, -1.4);
				
				\node[draw=none] at  (-1.5,-3-1.5-1.5) (a1) {\textbf{\large\color{darkgreen}src}};
				\draw[thick,->] (a1) -- (-1.5,-3-2);
				\node[draw=none] at  (4.5,-1.5+1.5) (a2) {\textbf{\large\color{darkgreen}dst}};
				\draw[thick,->] (a2) -- (4.5,-1);
			\end{tikzpicture}\\[1ex]}}
	\caption{Road network used for our experiments. The blue-dotted route, corresponding to Route {\color{blue} $7$} in Table \ref{tab:results}, has the lowest expected excess delay, while the green-dotted ``detour'' route, corresponding to Route {\color{blue} $8$} in Table \ref{tab:results}, has the lowest distorted excess delay.}
	\label{fig:road-net}
\end{figure}
\begin{table}[h]
	\caption{Expected and distorted excess delays of the routes from node $12$ to node $4$ in in Figure \ref{fig:road-net}. Route $7$ corresponds to the blue-dotted path in Figure \ref{fig:road-net}, while Route $8$ corresponds to the green-dotted detour in the same figure.}
	\label{tab:results}
	\centering
	\begin{tabular}{|c|c|c|}
		\toprule 
		& \textbf{Expected excess delay }& \textbf{Distorted excess delay}\\
		& \textbf{$x_i\tr\hat\theta_{\text{off}} - B$} &  \textbf{$\mu_{x_i}(\hat\theta_{\text{off}})$}\\\midrule
		Route $1$ & $-11.03 $ & $-0.20$\\\midrule
		Route $2$ & $5.76 $ & $-0.13$\\\midrule
		Route $3$ & $27.24 $ & $ 0.07$\\\midrule
		Route $4$ & $-15.61 $ & $-0.25$\\\midrule
		Route $5$ & $21.62 $ & $-0.01$\\\midrule
		Route $6$ & $-12.94$ & $-0.27$\\\midrule
		Route $7$ & $\bm{-21.67}$  & $-0.29$\\\midrule
		Route $8$ & $58.88 $ & $\bm{-0.45}$\\\midrule
		Route $9$ & $15.59$ & $-0.04$\\\midrule
		Route $10$ & $5.89$ & $-0.09$\\\bottomrule
	\end{tabular}
	\end{table}

For computing a performance benchmark to evaluate OFUL/W-OFUL, we perform an experiment to learn a linear (additive) relationship between the delay of a route and the delays along each of its component lanes. This is reasonable, considering the fact the individual delays along each lane stabilize when the traffic network reaches steady state, and the delay incurred along any route $x$ would be centered around $x\tr\theta$.
For learning the linear relationship, we simulate $100000$ steps of the traffic simulator to collect the delays along any route of the source-destination pair. Using these samples, we perform ridge regression to obtain $\hat\theta_{\text{off}}$, i.e., we solve
\begin{align}
	\hat\theta_{\text{off}} = \argmin_{\theta}  \dfrac{1}{2} \sum\limits_{i=1}^{\tau} (c_i - \theta\tr x_i)^2 + \lambda \left\| \theta \right\|^2,
\end{align}
where $c_i$ is the $i$th delay sample corresponding to a route choice $x_i,$ $\tau$ is the total number of delay samples for the source-destination pair considered, and $\lambda$ is a regularization parameter. It is well known that $\hat \theta_{\text{off}} =A^{-1} b$, where 
$A =  \sum\limits_{i=1}^{\tau} x_i x_i\tr + \lambda I_d$
and $b =  \sum\limits_{i=1}^{\tau} c_{i} x_i,$ where $I_d$ denotes the $d$-dimensional identity matrix, with $d$ set to the number of lanes in the road network considered. 

We set the various parameters of the problem as well as bandit routing algorithms - OFUL and W-OFUL, as follows: Regularization parameter $\lambda = 1$, confidence $\delta=0.99$, and norm bound $\beta = 400$.
We observed that both OFUL and W-OFUL algorithms learn the parameter underlying the linear relationship (Eq. \eqref{eq:linban-depnoise}) to a good accuracy. 



Figure \ref{fig:road-net} shows the grid network used for our experiments, while Table \ref{tab:results} tabulates the expected and distorted excess delays for ten routes that connect node $12$ to node $4$. As illustrated in Figure \ref{fig:road-net}, the route with the lowest expected excess delay, i.e., the blue-dotted route, is one of the shortest paths connecting node $12$ to $4$. On the other hand, the route with lowest distorted excess delay, i.e., the green-dotted route, is a detour. 
The lower value of distorted excess delay for the longer green-dotted route
over the shorter blue-dotted route is
presumably due to the fact that the two routes differ in the variance
of the end-to-end delay. This leads to rare events being overweighted
by the weight function, ultimately making the former path more
appealing from the distorted excess delay viewpoint. The shorter route is preferred by OFUL algorithm, while the distortion-aware WOFUL algorithm finds the longer route.

\section{Conclusions}
\label{sec:conclusions}
We incorporated weight-distortion based criteria -- a generalization of expected value -- into regret minimization and best arm identification problems in $K$-armed and linear bandit settings. For both settings, we proposed algorithms and derived upper bounds on their performance. We also showed order-optimality of our proposed algorithms by providing matching lower bounds. In numerical experiments, we compared our algorithms with existing algorithms that deal with the expected reward setting, and showed that our algorithms outperform the existing algorithms w.r.t. weight-distorted rewards. Therefore, our distortion-aware bandit algorithms are likely to be useful in aiding human decision making under uncertainty in real-life applications.  

\bibliography{cpt_refs}
\bibliographystyle{IEEEtran}

\clearpage\newpage
\appendices

\section{Proof of weight-distorted value estimation concentration result for sub-Gaussian case}
\label{sec:subgauss-conc}
%
%
\subsection{Proof of Proposition~\ref{prop:sub-gaussian-concentration}:}
Let
\begin{align*}
e_n &= \int_0^{\infty} w\left(\Prob{Y>z}\right) dz- \int_0^{M_n} w\left(1- {F_n}(z)\right) dz\\
& = I_1 + I_2,
\end{align*}
where 
\begin{align*}
I_1= \int_{M_n}^{\infty} w\left(1- F(z)\right) dz, \quad I_2 = \int_0^{M_n} w\left(1-F(z)\right) dz- \int_0^{M_n} w\left(1- {F_n}(z)\right) dz.
\end{align*}
Using sub-Gaussianity, the first term on the RHS is bounded as follows:
\begin{align*}
I_1&=\int_{M_n}^\infty w\left(\Prob{Y > z}\right) dz  \le L \int_{M_n}^\infty \left(\Prob{Y > z}\right)^\alpha dz  \\
&\leq 8L \int_{M_n}^\infty \frac{z}{M_n} \exp(-\frac{\alpha z^2}{2\sigma^2}) dz = \frac{8L\sigma^2}{\alpha M_n}  \exp\left(-\frac{\alpha M_n^{2}}{2\sigma^2}\right).
\end{align*}
Using the concentration result for CPT-value estimation for bounded distributions, 
the second term is bounded as follows:
\begin{align*}
&\Prob{I_2 > \epsilon} \le \exp\left(-\frac{2 n\epsilon^{2/\alpha}}{(L M_n)^{2/\alpha}}\right).
\end{align*}
The statement above is equivalent to saying, w.p. $(1-\delta)$, 
\[I_2 \le L M_n \left(\frac{\log(1/\delta)}{2n}\right)^{\alpha/2}.\]
Plugging in the bounds on $I_1$ and $I_2$, w.p. $1-\delta$, we have
\begin{align}
e_n
\le \frac{8L\sigma^2}{\alpha M_n} \exp\left(-\frac{\alpha M_n^{2}}{2\sigma^2}\right) + L M_n \left(\frac{\log(1/\delta)}{2n}\right)^{\alpha/2}.
\end{align}
Substituting $M_n = \sigma\sqrt{\log n}$, we obtain 
\begin{align*}
e_n
\le   \frac{8L\sigma^2}{\alpha n^{\alpha/2}} + L \sqrt{\frac{\log n}{2}} \left(\frac{\log(1/\delta)}{2 n}\right)^{\alpha/2}  \textrm{ w.p. } 1-\delta.
\end{align*}
The statement above can be converted to a standard tail bound as follows:
\begin{align*}
\Prob{e_n > \epsilon} \le  \exp\left( -2 n \left(\frac{2}{L^2\log n}\right)^{\frac1{\alpha}} \left(\epsilon - \frac{8L}{\alpha n^{\alpha/2}} \right)^{\frac{2}{\alpha}}  \right).
\end{align*}

A similar bound can be obtained for 
\begin{align*}
\tilde e_n &= \int_0^{\infty} w\left(\Prob{Y<-z}\right) dz- \int_0^{M_n} w\left({F_n}(-z)\right) dz.
\end{align*}
The main claim follows by combining the bounds on $e_n$ and $\tilde e_n$. 

\section{Proofs for K-armed bandits}
\label{sec:proofs-karmed}
In this section, we present the detailed proofs of analytical results presented in Section~\ref{sec:algos-karmed}. 

\subsection{Proofs for regret minimization setting}
\subsubsection{Proof of Theorem \ref{thm:BasicHolderRegretGap}}
\label{sec:appendix-gapdependentregret}
As in the case of regular UCB, we bound the number of times a sub-optimal arm is pulled using the technique from~\cite{munos2014survey}. 
Recall that, wlog we assumed that arm 1 is the optimal in Section~\ref{sec:model-karmed}. Suppose that a sub-optimal arm $k$ is pulled at time $m$, which implies that 
$$\widehat \mu_{k,T_k(m-1)} + \gamma_{m,T_k(m-1)} \ge  \widehat \mu_{1,T_{1}(m-1)} + \gamma_{m,T_{1}(m-1)}.$$
The UCB-value of arm $k$ can be larger than that of arm $1$ {\em only if} one of the following three conditions holds:
\begin{align}
\widehat \mu_{1,T_{1}(m-1)}  & \le \mu_*- \gamma_{m,T_{1}(m-1)}, \label{eq:optarmmeanoutside}\\
&\text{ or } \nonumber\\
\widehat \mu_{k,T_k(m-1)} & \ge \mu_k + \gamma_{m,T_k(m-1)}, \label{eq:karmmeanoutside}\\
&\text{ or } \nonumber\\
\quad \mu_* &< \mu_k \,+\, 2 \,\gamma_{m,T_{k}(m-1)}. \label{eq:cond3}  
\end{align}

Note that condition (\ref{eq:cond3}) above is equivalent to requiring 
\[T_{k}(m-1) \le  \dfrac{2^{2/\alpha -1} 3 (LM)^{2/\alpha}\log m}{\Delta_k^{2/\alpha}}.\]

Let $\kappa = \dfrac{2^{2/\alpha -1} 3 (LM)^{2/\alpha}\log n}{\Delta_k^{2/\alpha}} + 1$. When $T_{k}(m-1) \ge  \kappa$, i.e., the condition in \eqref{eq:cond3} does not hold, then either (i) arm $k$ is not pulled at time $m$, or 
(ii) \eqref{eq:optarmmeanoutside} or \eqref{eq:karmmeanoutside} occurs.
Thus, we have  
\begin{equation}
T_k(n) \le \kappa + \sum_{m=\kappa+1}^n I(I_m=k; T_k(m) > \kappa) \le \kappa + \sum_{m=\kappa+1}^n I(\eqref{eq:optarmmeanoutside} \text{ or }\eqref{eq:karmmeanoutside} \text{ occurs}). \label{eq:tkn}
\end{equation}
From Theorem \ref{thm:cpt-est}, we can upper bound the probability of occurence of \eqref{eq:optarmmeanoutside} as follows:
$$ \mathbb{P} \left( \exists l \in \lbrace 2, 3, \dots, m-1 \rbrace \text{ such that } \widehat \mu_{1,T_{1}(l-1)}  \le \mu_*- \gamma_{l,T_{1}(l-1)} \right) \le \sum_{l=1}^m \frac{2}{m^{3}} \le \frac{2}{m^{2}}.$$
A similar argument works for \eqref{eq:karmmeanoutside}.  
Plugging the bounds on the events governed by \eqref{eq:optarmmeanoutside} and \eqref{eq:karmmeanoutside} into \eqref{eq:tkn} and taking expectations, we obtain
  \begin{equation*}
	\E[T_k(n)] \leq \kappa + 2 \sum_{m=\kappa+1}^n \frac{2}{m^{2}} \le \dfrac{2^{2/\alpha -1} 3 (LM)^{2/\alpha}\log n}{\Delta_k^{2/\alpha}} + \left(1 + \dfrac{2\pi^2}{3} \right).
	\end{equation*}
  The claim follows by plugging the inequality above into the definition of expected regret, and using the fact that $\Delta_k \leq M$ for all arms $k$, as it follows that the weight-distorted reward of any distribution bounded by $M$ again admits an upper bound of $M$. \hfill$\blacksquare$

\begin{remark}
In the above proof, we have used simple union bounds for handling the random stopping times in the interest of readability. Note that improved techniques in the literature such as peeling arguments~\cite{bubeck2010jeux}, self-normalized concentration inequalities\cite{de2004self, bercu2008exponential} can also be used to obtain potentially better bounds. 
\end{remark}

\subsubsection{Proof of Corollary \ref{cor:BasicHolderRegretNoGap}}
\label{sec:appendix-gapindependentregret}
  We have, by the analysis in the proof of Theorem
  \ref{thm:BasicHolderRegretGap}, that for all $k$ with
  $\Delta_k > 0$,
  $$ \E[T_k(n)] \leq \dfrac{3(2LM)^{2/\alpha}\log n}{2\Delta_k^{2/\alpha}} +\left(1 + \dfrac{2\pi^2}{3} \right).
$$
  Thus, we can write
  \begin{align*}
    \E R_n &= \sum_{k} \Delta_k \; \E [T_k(n)] = \sum_{k} \left( \Delta_k \; \E [T_k(n)]^{\frac{\alpha}{2}}  \right) \left( \E [T_k(n)]^{1-\frac{\alpha}{2}} \right) \\
    &\leq \left( \sum_k \Delta_k^{2/\alpha} \; \E [T_k(n)] \right)^{\frac{\alpha}{2}} \left( \sum_k \E [T_k(n)] \right)^{1 - \frac{\alpha}{2}} \\
    &\text{(H\"{o}lder's inequality with exponents $2/\alpha$ and $2/(2-\alpha)$)} \\
    & \leq \left( \dfrac{3K(2LM)^{2/\alpha}\log n}{2} + K M^{2/\alpha}\left(1 + \dfrac{2\pi^2}{3} \right)  \right)^{\frac{\alpha}{2}} n^{\frac{2-\alpha}{2}}, \quad \quad \text{(using $\Delta_k^{2/\alpha} \leq M^{2/\alpha}$)}. 
  \end{align*}
  This proves the result.\hfill$\blacksquare$

\subsubsection{Proof of Theorem~\ref{thm:karmedlowerbd}} 
\label{sec:appendix-regretlowerbound}
We will need the following definition to prove the theorem. For two probability distributions $p$ and $q$ on $\mathbb{R}$, define the {\em relative entropy} or {\em Kullback-Leibler (KL) divergence} $D(p || q)$ between $p$ and $q$ as 
\[ D(p || q) = \int \log \left(\frac{dp}{dq}(x)\right) dp(x) \]
if $p$ is absolutely continuous w.r.t. $q$ (i.e., $q(A) = 0 \Rightarrow p(A) = 0$ for all Borel sets $A$), and $D(p || q) = \infty$ otherwise. For instance, if $p$ and $q$ are Bernoulli distributions with parameters $a \in (0,1)$ and $b \in (0,1)$, respectively, then a simple calculation gives that $D(p || q) = a \log \frac{a}{b} + (1-a) \log \frac{1-a}{1-b}$. 
We will often use $D(F || G)$ to mean the KL divergence between distributions represented by their cumulative distribution functions $F$ and $G$, respectively. 
\begin{proof}
We first show the result for rewards bounded by $M = 1$, the extension to general $M$ follows. 

The key ingredient in the proof is the seminal lower-bound on the number of suboptimal arm plays derived by Lai and Robbins \cite{lai1985lowerbd}. Their result shows that for any algorithm that plays suboptimal arms only a sub-polynomial number of times in the time horizon, 
\[ \liminf_{n \to \infty} \frac{\expect{T_k(n)}}{\log n} \geq \frac{1}{D(F_k || F_{1})} \]
for any suboptimal arm $k$ and we assumed that arm~1 is the optimal arm. (Note that the argument at its core uses only a change-of-measure idea and the sub-polynomial regret hypothesis, and is thus unaffected by the fact that we measure regret by distorted, i.e., non-expected, rewards.)

Using this along with the definition of regret, we get
\begin{align}
	\liminf_{n \to \infty} \frac{\expect{R_n}}{\log n} &= \liminf_{n \to \infty} \frac{\sum\limits_{k=2}^K \E[T_k(n)] \Delta_k}{\log n} \nonumber \\
	&\geq \sum\limits_{k=2}^K \Delta_k \cdot \liminf_{n \to \infty} \frac{\E[T_k(n)] }{\log n} \nonumber \\
	&\geq \sum\limits_{k=2}^K \frac{\Delta_k}{D(F_k || F_{1})} \nonumber \\
	&= \sum\limits_{k=2}^K \dfrac{(2LM)^{2/\alpha}}{\Delta_k^{2/\alpha - 1}} \cdot \frac{(\Delta_k/2LM)^{2/\alpha}}{D(F_k || F_{1})}. \label{eq:liminfbound}
\end{align}

(Recall that $\Delta_k = \mu_{1} - \mu_k$ represents the difference between the weight-distorted values of arms $k$ and $1$.) \\

We now design a set of reward distributions for the arms, for which the limiting property claimed in the theorem holds. In fact, we will show that it is enough to design Bernoulli distributions for the arms' rewards, i.e., arm $k$'s reward is Ber($p_k$), or equivalently, $F_k(x) = (1-p_k)\mathbf{1}_{[0,1)}(x) + \mathbf{1}_{[1,\infty)}(x)$. This gives a simple expression for the weight-distorted reward of any arm $k$ as $\mu_k = \int_0^1 w(p_k) dz = w(p_k)$, and, consequently, $\Delta_k = w(p_{1}) - w(p_k)$, for any monotonic increasing weight function $w: [0,1] \to [0,1]$ with $w(0) = 1 $ and $w(1) = 1.$

Consider now a weight function $w:[0,1] \to [0,1]$ which is monotone increasing from $0$ to $1$, H\"{o}lder-continuous with any desired constant $L > 0$ and exponent $\alpha \in (0,1]$, and, moreover, satisfies the following ``\holder continuity property from below'': for some $\tilde{p} \in (0,1)$ (say, $\tilde{p} = 1/2$), $|w(\tilde{p}) - w(p)| \geq L |\tilde{p} - p|^\alpha$. Such a weight function can always be constructed, e.g., for $\alpha = 1/2, L = 1/\sqrt(2),$ take the function $w(x) = \frac{1}{2} - \frac{1}{\sqrt{2}}\sqrt{\frac{1}{2} - x}$ for $x \in [0, 1/2]$, and $w(x) = \frac{1}{2} + \frac{1}{\sqrt{2}}\sqrt{x - \frac{1}{2}}$ for $x \in (1/2, 1]$. (This is essentially formed by gluing together two inverted and scaled copies of the function $\sqrt{x}$, to make an S-shaped function infinitely steep at $x = 1/2$.)

For such a weight function, putting $p_{1} = \tilde{p} = 1/2$ and $p_k < p_{1}$, $k \neq 1$, we can use the standard bound\footnote{This results from using $\log x \leq x - 1$.} 
\begin{align*} 
D(F_k || F_{1}) &= p_k \log \frac{p_k}{p_{1}} + (1-p_k) \log \frac{1-p_k}{1-p_{1}}  \\
&\leq \frac{(p_{1} - p_k)^2}{p_{1}(1 - p_{1})} = 4{( p_{1} - p_k)^2} \\
&\leq 4 \left( \frac{ w(p_{1}) - w(p_k) }{L} \right)^{2/\alpha} \quad \quad \mbox{(by the ``lower-\holder'' property of $w$)} \\
&= 4 \left( \frac{\Delta_k}{L} \right)^{2/\alpha}. 
\end{align*}
Since $M = 1$ for the Bernoulli family of reward distributions, this implies, together with \eqref{eq:liminfbound}, that 
\[ \liminf_{n \to \infty} \frac{\expect{R_n}}{\log n} \geq  \sum_{\{k: \Delta_k > 0\}} \dfrac{L^{2/\alpha}}{4\Delta_k^{2/\alpha - 1}}. \]
For general $M$, one can modify this construction and consider arms' reward distributions to be Bernoulli scaled by the constant $M$ (i.e., equal to $M$ w.p. $p$ and $0$ otherwise). This completes the proof. 
\end{proof}

\subsection{Proofs for best arm identification setting: Fixed confidence}
For proving the claims in Theorems~\ref{Thm:correctness} and~\ref{Thm:LUCB6}, we adopt the approach from \cite{kalyanakrishnan2012pac}. We use the concentration inequality presented in Theorem~\ref{thm:cpt-est} rather than the Hoeffding's inequality used in the analysis of LUCB algorithm in~\cite{kalyanakrishnan2012pac}; this is due to the fact that the rate at which the weight-distored reward estimate concentrates is slightly slower than the rate at which sample mean concentrates around true mean.
In particular, one would require $O\left(1/\epsilon^{2/\alpha}\right)$ samples (due to Theorem~\ref{thm:cpt-est}) to get an estimate of the weight-distorted reward that is $\epsilon$-accurate, while the expected value can be estimated to $\epsilon$ accuracy using $O\left(1/\epsilon^{2}\right)$ samples. Note that the sample complexity in Theorem \ref{thm:cpt-est} is not sub-optimal, as the authors show via information-theoretic lower bounds in Proposition 5 of \cite{prashanth2017cptlong}.
\subsubsection{Proof of Theorem~\ref{Thm:correctness}.}
\label{sec:proof-of-W-LUCB-correctness}
We follow the proof technique from~\cite{kalyanakrishnan2012pac}. Without loss of generality, we assume that arm~$1$ is the optimal arm. 
\begin{align*}
&\mathbb{P} \left( \text{W-LUCB does not return optimal arm} \right) = \\ &\sum\limits_{m=1}^\infty \mathbb{P} \left( \text{W-LUCB terminates in round-m and does not return optimal arm} \right) \\
& \sum\limits_{m=1}^\infty \sum\limits_{b=2}^K \mathbb{P} \left( \text{W-LUCB terminates in round-~$m$ and returns arm~$b$} \right)
\end{align*}
Along the lines of proof of Theorem~1 in~\cite{kalyanakrishnan2012pac}, we obtain that  
\begin{multline}
\label{eq:abi-1}
\mathbb{P} \left( \text{W-LUCB terminates in round-$m$ and returns arm~b} \right) \le \\
\mathbb{P} \left( \widehat{\mu}_{1,m} + \gamma'_{m, T_1(m-1)} \le \mu_1 \right) + \mathbb{P} \left( \widehat{\mu}_{b,m} - \gamma'_{m, T_b(m-1)} \ge \mu_b \right).
\end{multline}
The first term on the RHS above can be bounded as follows:
\begin{equation*}
\mathbb{P} \left( \widehat{\mu}_{1,m} + \gamma'_{m, T_1(m-1)} \le \mu_1 \right) \le \sum\limits_{s=1}^m \mathbb{P} \left( \widehat{\mu}_{1,m} + \gamma'_{m, s} \le \mu_1 \right) \le \sum\limits_{s=1}^m \exp \left( -2s \left( \frac{\gamma'_{m,s}}{LM} \right)^{2/\alpha} \right),
\end{equation*}
where the final inequality follows from Theorem \ref{thm:cpt-est}. Similarly, we can upper bound the second term in the RHS of~\eqref{eq:abi-1}. 
Thus,
\begin{align*}
\mathbb{P} \left( \text{W-LUCB does not return optimal arm} \right) 
\le K \sum\limits_{m=1}^\infty \sum\limits_{s=1}^m 2\exp \left( -2s \left( \frac{\gamma'_{m,s}}{LM} \right)^{\frac{2}{\alpha}} \right) < \delta,
\end{align*}
where the final inequality follows from the fact that 
$\gamma'_{m,s} = LM \left[ \frac{1}{2s} \log \left( \frac{m^4 K \left(\frac{\pi^2}{3} + 1\right)}{\delta} \right) \right]^{\frac{\alpha}{2}}$.

\hfill$\blacksquare$
\subsubsection{Proof of Theorem~\ref{Thm:LUCB6}.}
\label{sec:proof-W-LUCB-complexity-bound}
Recall that $N^{CL}$ denotes the time instant when W-LUCB terminates, and Theorem~\ref{Thm:LUCB6} gives an upper bound on $\mathbb{E} \left[ N^{CL} \right]$. Let $Term_m$ denote the event that the algorithm terminates in a round $m$. Then the expected sample complexity is given by
\begin{align}
\mathbb{E}[N^{CL}] = \mathbb{E} \left[ \sum\limits_{m} \mathbb{I} \lbrace \sim Term_m \rbrace \right] = \sum\limits_m \mathbb{P} \left( \sim Term_m \right).
\end{align}
We establish Theorem~\ref{Thm:LUCB6} by deriving an upper bound on $\mathbb{P} \left( \sim Term_m \right).$  
Before we proceed further, we introduce the required notation. Let $ c = \frac{\mu_1 + \mu_2}{2}.$ It is easy to see that $\frac{\Delta_k}{2} \le \vert c - \mu_k \vert \le \Delta_k$ for any arm $k \in \mathcal{A} = \lbrace 1, 2, \cdots, K \rbrace.$ In round-$m,$ define
\begin{align*}
Above_m &= \lbrace k \in \mathcal{A} : \underbrace{\hat{\mu}_{k,T_k(m-1)} - \gamma'_{m,T_k(m-1)}}_{LCB_k} > c \rbrace, \\
 Below_m &= \lbrace k \in \mathcal{A} :  = \underbrace{\hat{\mu}_{k,T_k(m-1)} + \gamma'_{m,T_k(m-1)}}_{UCB_k} < c \rbrace, \text{and} \\
 Needy_m & = \mathcal{A} \setminus \left( Above_m \cup  Below_m \right)
\end{align*}
Intuitively, we expect the optimal arm to lie in $Above_m$ or $Needy_m$ and all sub-optimal arms to lie in $Below_m$ or $Needy_m.$ Now, we define an event $Cross_m^k$ which captures the scenario when the aforementioned condition does not occur, i.e., 
\begin{align*}
Cross^k_m =
\begin{cases}
k \in Below_m \hspace{0.45cm} \text{if arm~$k$ is the optimal arm},\\
k \in Above_m \hspace{0.45cm} \text{if arm~$k$ is a sub-opimal arm}
\end{cases}
\end{align*}
Let $Cross_m$ be the event when at least one arm~$k$ satisfies $Cross^k_m$, i.e., 
\begin{align}
Cross_m = \exists k \in \mathcal{A} : Cross^k_m \,\,\text{occured}
\end{align}

We now state three lemmas that are useful in bounding $\mathbb{P} \left( \sim Term_m \right)$, which in turn assists in bounding the sample complexity. 
\begin{lemma}
\label{Thm:LUCB2}
For any $m > 1,$ W-LUCB satisfies the following:
\begin{align*}
\sim Cross_m \cap \sim Term_m \subseteq \lbrace h_m \in Needy_m \rbrace \cup \lbrace l_m \in Needy_m \rbrace,
\end{align*}
where $h_m$ and $l_m$ are defined in the Algorithm~\ref{alg:cpt-lucb}.
\end{lemma}
\begin{proof}
It directly follows from Lemma~2 in~\cite{kalyanakrishnan2012pac}.
\end{proof}
The above lemma shows that if $Cross_m$ does not occur and the W-LUCB does not terminate, then either arm $h_m$ or arm $l_m$ is in $Needy_m.$ We also infer that if the algorithm does not terminate in round-$m,$ then either $Cross_m$ occurs or $h_m$ or $l_m$ is in $Needy_m.$ In the next two lemmas, we provide upper bounds on $\mathbb{P} \left( Cross_m \right)$ and $\mathbb{P}  \left( \lbrace h_m \in Needy_m \rbrace \cup \lbrace l_m \in Needy_m \rbrace \right).$  
\begin{lemma}
\label{Thm:LUCB3}
For any $m > 1,$ W-LUCB algorithm satisfies 
\begin{align*}
\mathbb{P} \left( Cross_m \right) \le \frac{2 \delta}{a m^3}, \textrm{ where }\delta \in (0, 1) \textrm{ and }a = \pi^2/3 +1.
\end{align*}
\end{lemma}
\begin{proof}
Assume that arm~1 is the optimal arm. 
\begin{align*}
\mathbb{P} \left( Cross_m \right) = \mathbb{P} \left( \cup_{k \in \mathcal{A}} Cross_m^k \right) \le \sum\limits_{k \in \mathcal{A}} \mathbb{P} \left( Cross_m^k \right) \le \mathbb{P} \left( 1 \in Below_m \right) + \sum\limits_{k=2}^K \mathbb{P} \left( k \in Above_m \right).
\end{align*}
Now, we upper bound $\mathbb{P} \left( 1 \in Below_m \right)$ as follows. 
\begin{align*}
\mathbb{P} \left( 1 \in Below_m \right) & = \mathbb{P} \left( \widehat{\mu}_{1,m} + \gamma'_{m, T_1(m-1)} < c \right) \le \sum\limits_{s=1}^m \mathbb{P} \left( \widehat{\mu}_{1,m} + \gamma'_{m, s} < c \right) \\
& = \sum\limits_{s=1}^m \mathbb{P} \left( \widehat{\mu}_{1,m}  - \mu_1 < c -\mu_1 - \gamma'_{m, s} \right) \le \sum\limits_{s=1}^m 2 \exp \left( -2s \left( \frac{\mu_1 - c + \gamma'_{m,s}}{LM} \right)^{2/\alpha} \right) \\
&\overset{(*)}{\le} \sum\limits_{s=1}^m 2 \exp \left( -2s \left( \frac{\gamma'_{m,s}}{LM} \right)^{2/\alpha} \right) 
= \frac{2 \delta}{K m^3 a},
\end{align*}
where $(*)$ is due to the fact that $(\mu_1 -c) > 0$. Similarly, we can prove that $\mathbb{P} \left( k \in Above_m \right) \le \frac{2 \delta}{K m^3 a}$ for $2 \le k \le K,$ which completes the proof.
\end{proof}
\begin{lemma}
\label{Thm:LUCB4}
Let $d_\alpha = \frac{1}{\left( 1 - \left( \frac{1}{4} \right)^{\alpha/2} \right)^{2/\alpha}}.$ For any arm~$k$ and round-$m,$ let  \\ $\beta \left( k, m \right) = \lceil \frac{1}{2 d_\alpha} \left( \frac{2LM}{\Delta_k}\right)^{2/\alpha} \log \left( \frac{m^4 K a}{\delta} \right) \rceil.$
Then, we have 
\begin{align}
\mathbb{P} \left( \exists k \in \mathcal{A} : T_k(m-1) > 4 \beta (k,m) \cap \lbrace k \in Needy_m \rbrace \right) 
\le \frac{H_\alpha d_\alpha \delta}{2m^4 K a}, \label{eq:LCB4claim}
\end{align}
where $H_\alpha$ is as defined in Theorem~\ref{Thm:LUCB6}.
\end{lemma}
\begin{proof}
First, we upper bound the LHS of \eqref{eq:LCB4claim} for an arm~$k$ and then use a union bound. 
Consider the case when $k = 1$. We have,
\begin{align*}
&\mathbb{P} \left( T_1(m-1) > 4 \beta (1,m) \cap \lbrace 1 \in Needy_m \rbrace \right) 
\le \mathbb{P} \left( T_1(m-1) > 4 \beta (1,m) \cap \lbrace 1 \in \sim Above_m \rbrace \right) \\
& = \mathbb{P} \left( T_1(m-1) > 4 \beta (1,m) \cap \hat{\mu}_{1,T_1(m-1)} - \gamma'_{m,T_1(m-1)} < c \right)\\
& = \sum\limits_{s = 4\beta(1,m)+1}^\infty \mathbb{P} \left( \widehat{\mu}_{1,s}-\mu_1 < -( \mu_1 -c - \gamma'_{m,s}) \right) \le \sum\limits_{s = 4\beta(1,m)+1}^\infty \exp \left( -2s \left( \frac{\frac{\Delta_1}{2} -\gamma'_{m,s}}{LM} \right)^{2/\alpha} \right) \\
&  = \sum\limits_{s = 4\beta(1,m)+1}^\infty \exp \left( -2s \left( \frac{\Delta_1}{2LM} - \left[ \frac{1}{s} \left( \frac{\Delta_1}{2LM} \right)^{2/\alpha} \beta(1,m) \right]^{\alpha/2} \right)^{2/\alpha} \right) \\
& = \sum\limits_{s = 4\beta(1,m)+1}^\infty \exp \left( -2 \left( \frac{\Delta_1}{2LM}\right)^{2/\alpha} \left[ s^{\alpha/2} - \beta(1,m)^{\alpha/2} \right]^{2/\alpha} \right) \\
& \le \int\limits_{s = 4\beta(1,m)}^\infty \exp \left( -2s \left( \frac{\Delta_1}{2LM}\right)^{2/\alpha} \left[ 1 - \left( \frac{\beta(1,m)}{s} \right)^{\alpha/2} \right]^{2/\alpha} \right) ds \\
& \le \int\limits_{s = 4\beta(1,m)}^\infty \exp \left( -2s \left( \frac{\Delta_1}{2LM}\right)^{2/\alpha} \left[ 1 - \left( \frac{1}{4} \right)^{\alpha/2} \right]^{2/\alpha} \right) ds \\
& \le \frac{\delta}{2m^4 K a} \left( \frac{2LM}{\Delta_1} \right)^{2/\alpha} \frac{1}{\left( 1 - \left( \frac{1}{4} \right)^{\alpha/2} \right)^{2/\alpha} }.
\end{align*}
Similarly, it can be shown that for any arm $2 \le k \le K,$
\begin{align*}
\mathbb{P} \left( T_k(m-1) > 4 \beta (k, m) \cap \lbrace k \in Needy_m \rbrace \right) \le  \frac{\delta}{2m^4 K a} \left( \frac{2LM}{\Delta_k} \right)^{2/\alpha} \frac{1}{\left( 1 - \left( \frac{1}{4} \right)^{\alpha/2} \right)^{2/\alpha} }
\end{align*}
From the foregoing,
\begin{align*}
\mathbb{P} \left( \exists k \in \mathcal{A} : T_k(m-1) > 4 \beta (k,m) \cap \lbrace k \in Needy_m \rbrace \right) \le \frac{H_\alpha d_\alpha \delta}{2m^4 K a}. 
\end{align*}
This proves the lemma.
\end{proof}
In the following result, we upper bound $\mathbb{P} \left( \sim Term_m \right)$ for a sufficiently large~$m,$ using the lemmas stated above.
\begin{lemma}
\label{Thm:LUCB5}
For any $m > 150 H_\alpha \log \left( \frac{H_\alpha}{\delta} \right),$ W-LUCB satisfies
$\mathbb{P} \left( \sim Term_m \right) = O \left( \frac{1}{m^2} \right).$
\end{lemma}
\begin{proof}
We first prove the lemma by fixing an $m$ that is larger than $ 150 H_\alpha \log \left( \frac{H_\alpha}{\delta} \right),$ which in turn implies that the lemma holds for any $m > 150 H_\alpha \log \left( \frac{H_\alpha}{\delta} \right)$. Fix 
\begin{align}
\label{eq:fix_m_equation}
m = c H_\alpha \log \left( \frac{H_\alpha}{\delta} \right), \text{where } c > 150.
\end{align}
Let $\bar{m} = \lceil \frac{m}{2} \rceil.$ Define the following events:
\begin{align*}
E_1 &:= \exists s \in \lbrace \bar{m}, \bar{m}+1, \cdots, m-1 \rbrace: Cross_m  \\
E_2 &:=  \exists s \in \lbrace \bar{m}, \bar{m}+1, \cdots, m-1 \rbrace, \exists k \in \mathcal{A}: \left( T_k(m-1) > 4\beta (k,m) \cap \lbrace k \in Needy_m \rbrace \right)
\end{align*}
Now, we show that under the event $\sim E_1 \cap \sim E_2,$ W-LUCB algorithm terminates in at most $m$ rounds. Suppose the algorithm terminates in $\bar{m}$ rounds, then there is nothing to prove. Assume that the algorithm has not terminated in first $\bar{m}$ rounds. Let $AR$ be the additional number of rounds the algorithm lasts from $\bar{m}$ till $m$, under the event $\sim~E_1~\cap~\sim~E_2.$ We show that $\bar{m} + AR < m$ which implies that the algorithm has terminated in at most $m$ rounds under the considered event. Due to the fact that the algorithm has not terminated, $\sim E_1$, and Lemma~\ref{Thm:LUCB2}, 
\begin{align*}
AR  = \sum\limits_{s = \bar{m}+1}^m \mathbb{I} \lbrace h_s \in Needy_s\,\, \text{or} \, \, l_s \in Needy_s \rbrace & \le \sum\limits_{s = \bar{m}+1}^m \sum\limits_{k=1}^K \mathbb{I} \lbrace k \in Needy_s  \cap \left( k = h_s \,\, \text{or}\,\, k = l_s \right) \rbrace \\
& \hspace{-2cm} \overset{(a)}{=} \sum\limits_{s = \bar{m}+1}^m \sum\limits_{k=1}^K \mathbb{I} \lbrace \left( k = h_s \,\, \text{or}\,\, k = l_s \right) \cap T_k(s-1) \le 4 \beta (k,s) \rbrace \\
& \hspace{-2cm} \overset{(b)}{\le} \sum\limits_{s = \bar{m}+1}^m \sum\limits_{k=1}^K \mathbb{I} \lbrace \left( k = h_s \,\, \text{or}\,\, k = l_s \right) \cap T_k(s-1) \le 4 \beta (k,m) \rbrace \\
& \hspace{-2cm} = \sum\limits_{k=1}^K\sum\limits_{s = \bar{m}+1}^m \mathbb{I} \lbrace \left( k = h_s \,\, \text{or}\,\, k = l_s \right) \cap T_k(s-1) \le 4 \beta (k,m) \rbrace \\
& \hspace{-2cm} \le \sum\limits_{k=1}^K 4 \beta \left( k, m \right),
\end{align*}
where $(a)$ is due to $\sim E_2$ and  $(b)$ is due to the fact that $\beta(k,s)$ is an increasing function in~$s.$ Now, we bound $2 + 8 \sum\limits_{k=1}^K 4 \beta(k,m)$ is as follows.
\begin{align*}
2 + 8 \sum\limits_{k=1}^K 4 \beta(k,m)  & \le 2 + 8 \sum\limits_{k} \left( \frac{1}{2d_\alpha} \left( \frac{2LM}{\Delta_k} \right)^{2/\alpha} \log \left( \frac{m^4 K a}{\delta} \right) +1 \right) \\
& = 2 + 8K + \frac{4}{d_\alpha} H_\alpha \log \left( \frac{m^4 K a}{\delta} \right) \\
&  \overset{(a)}{<} 2K + 8K + 4H_\alpha \log \left( \frac{m^4 K a}{\delta} \right) \\
& = 10K + 16 H_\alpha \log m + 4 H_\alpha \log \left( \frac{K}{\delta} \right) + 4 H_\alpha \log a \\
& \overset{(b)}{\le} (10 + 4 \log a) H_\alpha + 4 H_\alpha \log \left( \frac{H_\alpha}{\delta} \right) + 16 H_\alpha \log m \\
&  \overset{(c)}{=} (10 + 4 \log a) H_\alpha + 4 H_\alpha \log \left( \frac{H_\alpha}{\delta} \right) + 16 H_\alpha \left( \log c + \log H_\alpha + \log \log \frac{H_\alpha}{\delta} \right) \\
& \overset{(d)}{\le} (10 + 4 \log a) H_\alpha + 4 H_\alpha \log \left( \frac{H_\alpha}{\delta} \right) + 16 H_\alpha \left( \log c + \log H_\alpha + \log \frac{H_\alpha}{\delta} \right) \\
& \overset{(e)}{\le} 18 H_\alpha + 36 H_\alpha \log \left( \frac{H_\alpha}{\delta} \right) + 16 H_\alpha \log c \\
& \overset{(f)}{\le} ( 54 + 16 \log c) H_\alpha \log \left( \frac{H_\alpha}{\delta} \right) \\
& \overset{(g)}{<} c H_\alpha \log \left( \frac{H_\alpha}{\delta} \right) = m 
\end{align*}
where $(a)$ is due to $K > 1,$ $(b)$ is due to $K \le H_\alpha,$  $(c)$ is due to ~\eqref{eq:fix_m_equation}, $(d)$ is due to $\log \log \frac{H_\alpha}{\delta} \le \log \frac{H_\alpha}{\delta},$ $(e)$ is due to $H_\alpha \le H_\alpha/\delta,$ $\log a \le 2$, $(f)$ is due to $ \log \frac{H_\alpha}{\delta} \ge 1$, and $(g)$ is due to the fact that $(54+16\log c)< c$ for any $c > 150$. \\
We see that $2(1 + AR) < m \Rightarrow AR < m/2 - 1.$ Hence, the total number of rounds the W-LUCB algorithm lasts is at most $\bar{m} + AR \le m/2 + 1 + AR < m.$ This proves that the W-LUCB algorithm terminates in at most $m$ rounds under the event $\sim E_1 \cap \sim E_2.$ Therefore, for $ m > 150 H_\alpha \log \frac{H_\alpha}{\delta}$ 
\begin{align}
\mathbb{P} \left( \text{W-LUCB has not terminated in round-m} \right) = \mathbb{P} \left( E_1 \cup E_2 \right).
\end{align}
Using Lemmas~\ref{Thm:LUCB3} and \ref{Thm:LUCB4}, we upper bound $\mathbb{P} \left( E_1 \cup E_2 \right)$ as follows:
\begin{align*}
\mathbb{P} \left( E_1 \cup E_2 \right) &\le \sum\limits_{s = \bar{m}}^m \left( \frac{2\delta}{as^3} + \frac{H_\alpha d_\alpha\delta}{2 s^4 K a} \right) \le \frac{m}{2} \left( \frac{16 \delta}{a m^3} + \frac{8 d_\alpha H_\alpha \delta}{aKm^4} \right) \le \frac{m}{2} \left( \frac{16 \delta}{a m^3} \right) \left( 1 + \frac{d_\alpha H_\alpha}{2Km} \right) \\
& = O \left( \frac{1}{m^2} \right)
\end{align*}
which establishes the result.
\end{proof}
It is easy to see from the result above that, w.p. at least $(1- \delta)$, the W-LUCB algorithm terminates after $O \left( H_\alpha \log \frac{H_\alpha}{\delta} \right)$ rounds. 

With the aid of the above result, we now present the proof of Theorem~\ref{Thm:LUCB6}.
\begin{proof}
Recall that, for any $m > K,$ the W-LUCB algorithm pulls two arms in any round. From Lemma~\ref{Thm:LUCB5}, we have
\begin{align*}
\mathbb{E} (\text{Sample complexity} ) \le &  2 \times 150 H_\alpha \log \left( \frac{H_\alpha}{\delta} \right) + \sum\limits_{m > 150 H_\alpha \log \left( \frac{H_\alpha}{\delta} \right)} O \left( \frac{1}{m^2} \right) \\
& = 300 H_\alpha \log \left( \frac{H_\alpha}{\delta} \right) + \text{constant} = O \left( H_\alpha \log \frac{H_\alpha}{\delta} \right),
\end{align*}
which completes the proof.
\end{proof}
\subsubsection{Proof of Theorem~\ref{Thm:FCSBLB}}
\label{sec:proof-complexity-lower-bound}
Note that, our proof uses an information-theoretic lower bound given in~\cite{kaufmann2015complexity}. In particular, we construct a particular set of arm distributions and a weight function as follows. For sake of convenience, we give a proof for $M = 1$ \emph{i.e.,} rewards are bounded by 1, and the similar analysis follows for general $M.$ Consider the $K$-armed bandit problem with reward distributions $Ber(p_1), Ber(p_2) \cdots, Ber(p_K).$ Without loss of generality, we assume that $p_1 > p_2 \ge p_3 \ge\cdots p_K.$ As shown in the proof of Theorem~\ref{thm:karmedlowerbd}, we can construct a H\"older-continuous weight function with the following property, for a given $L$ and $\alpha$: 
\begin{align}
\textrm{For some }\tilde{p} \in (0, 1), 
\vert w(p) - w(\tilde{p}) \vert \ge L \vert p - \tilde{p} \vert^\alpha, \forall p\in [0,1].
\label{eq:specialw}
\end{align}
Clearly, for a bernoulli r.v. with parameter $p_1$, weight-distorted reward equals $w(p_1)$. Assume that $p_1 = \frac{1}{2}$ and $p_k < \frac{1}{2}$ for any $2 \le k \le K.$ From Lemma~1 and Theorem~4 in~\cite{kaufmann2015complexity}, we obtain
\begin{align}
\mathbb{E}\left[ T_k(N_A) \right] & \ge \frac{1}{KL(p_2, p_k)} \cdot \frac{1}{\log(2.4 \delta)} \hspace{1cm} \text{for $k =1$}  \label{eq:Rav1}\\
\mathbb{E}\left[ T_k(N_A) \right] & \ge \frac{1}{KL(p_k, p_1)} \cdot \frac{1}{\log(2.4 \delta)} \hspace{1cm} \text{for $2 \le k \le K$} \label{eq:Rav2}
\end{align}  
Notice that, $\mathbb{E}[N_A] = \sum\limits_{k=1}^K \left[ T_k(\tau) \right].$ Now, we bound the KL-divergences as follows: For $k \ge 2$
\begin{align}
KL(p_k, p_1) &= p_k \log \left( \frac{p_k}{p_1} \right) + (1-p_k) \log \left( \frac{1-p_k}{1-p_1} \right) \overset{(a)}{\le} \frac{(p_1 - p_k)^2}{p_1 (1-p_1)} = 4(p_k - p_1)^2 \nonumber\\
& \overset{(b)}{\le} 4 \left( \frac{w(p_k) - w(p_1)}{L} \right)^{2/\alpha} = 4 \left( \frac{\Delta_k}{L} \right)^{2/\alpha} \label{eq:Rav3}
\end{align}
where $(a)$ is due to $\log x \le x-1$ and $(b)$ is due to the H\"older continuity property from the below. From \eqref{eq:Rav1}--\eqref{eq:Rav3}, we obtain
\begin{equation*}
\mathbb{E} [N_A] \ge \sum\limits_{k=1}^K 4 \left( \frac{L}{\Delta_k} \right)^{2/\alpha} \log \left( \frac{1}{2.4 \delta} \right) = O \left( H_\alpha \log \left( \frac{1}{2.4 \delta} \right) \right),
\end{equation*}
which completes the proof.
\hfill$\blacksquare$

\subsection{Proofs for best arm identification setting: Fixed budget}
\subsubsection{Proof of Theorem~\ref{Thm:SR1}}
\label{sec:proof-W-SR-error-upper-bound}
For establishing this result, we follow the proof technique from~\cite{audibert2010best}. Without loss of generality, we assume that arm $1$ is the optimal arm. Then, $J_n \neq 1$ implies that the optimal arm is rejected in one of the $K-1$ phases. Suppose that arm $1$ is rejected in phase $k$ of W-SR. Then, we have $\widehat{\mu}_{1, n_k} \le \hat{\mu}_{j, n_k} \forall j \in \mathcal{A}_k.$ Since the algorithm is in the phase $k$, $(k-1)$ sub-optimal arms must have been rejected in the first $(k-1)$ phases. Hence, at least one of the arms in $\lbrace
K+1-k, K+2-k, \cdots, K \rbrace$ must have survived until phase $k$ and this arm must be in $\mathcal{A}_k.$ Thus,
\begin{align*}
\lbrace \widehat{\mu}_{1,n_k} \le \widehat{\mu}_{j, n_k}, \forall j \in \mathcal{A}_k \rbrace \subseteq \lbrace \widehat{\mu}_{1,n_k} \le \widehat{\mu}_{a, n_k} : a \in \lbrace K+1-k, K+2-k, \cdots, K \rbrace, a\in \mathcal{A}_k \rbrace
\end{align*}
\begin{align*}
\mathbb{P} \left( \widehat{\mu}_{1,n_k} \le \widehat{\mu}_{j, n_k}, \forall j \in \mathcal{A}_k \right) & \le \sum\limits_{i=1}^k \mathbb{P} \left( \widehat{\mu}_{1,n_k} \le \widehat{\mu}_{K+i-k, n_k} \right) \mathbb{I} \lbrace \text{arm~$(K+i-k) \in \mathcal{A}_k$} \rbrace \\
& \le \sum\limits_{i=1}^k \mathbb{P} \left( \widehat{\mu}_{1,n_k} \le \widehat{\mu}_{K+i-k, n_k} \right).
\end{align*}
Therefore,
\begin{align*}
\mathbb{P} \left( J_n \neq 1 \right) \le \sum\limits_{k=1}^{K-1}\sum\limits_{i=1}^k \mathbb{P} \left( \hat{\mu}_{1, n_k} \le \hat{\mu}_{K+i-k, n_k} \right) &= \sum\limits_{k=1}^{K-1}\sum\limits_{i=K+1-k}^K \mathbb{P} \left( \hat{\mu}_{1, n_k} \le \hat{\mu}_{i, n_k} \right) \\
& = \sum\limits_{k=1}^{K-1}\sum\limits_{i=K+1-k}^K \mathbb{P} \left( \hat{\mu}_{1, n_k} - \hat{\mu}_{i, n_k} - (\mu_i - \mu_1) \ge \Delta_i \right) \\
& \overset{(a)}{\le} \sum\limits_{k=1}^{K-1} \sum\limits_{i=K+1-k}^K \exp \left( -2 n_k \left( \frac{\Delta_i}{LM} \right)^{2/\alpha} \right) \\
& \overset{(b)}{\le} \sum\limits_{k=1}^{K-1} k \exp \left( -2 n_k \left( \frac{\Delta_{K+1-k}}{LM} \right)^{2/\alpha} \right),
\end{align*}
where, $(a)$ follows from Theorem~\ref{thm:cpt-est} and $(b)$ from the fact that $\Delta_{K+1-k} \le \Delta_{K+i-k} \forall i \ge 1.$ Note that,
\begin{equation*}
n_k \left( \frac{\Delta_{K+1-k}}{LM}\right)^{2/\alpha}  \ge \frac{n-K}{\overline{\log} K} \frac{1}{K+1-k} \frac{1}{\left( \frac{\Delta_{K+1-k}}{LM} \right)^{-2/\alpha}} \ge \frac{n-K}{\overline{\log} K} \left( \frac{1}{M}\right)^{2/\alpha} \frac{1}{H_{2,\alpha}},
\end{equation*}
where $\overline{\log} K = \frac{1}{2} + \sum\limits_{i=2}^K \frac{1}{i}$.
Thus,
\begin{align*}
\mathbb{P} (J_n \neq 1) &\le \sum\limits_{k=1}^{K-1} k \exp \left( -\frac{n-K}{\overline{\log} K} \left( \frac{1}{LM}\right)^{2/\alpha} \frac{1}{H_{2,\alpha}} \right) \\
&= \frac{K(K-1)}{2} \exp \left( -\frac{n-K}{\overline{\log} K} \left( \frac{1}{LM}\right)^{2/\alpha} \frac{1}{H_{2,\alpha}} \right),
\end{align*}
which completes the proof.
\hfill$\blacksquare$
\subsubsection{Proof of Theorem~\ref{Thm:FBSBLB1}}
\label{sec:proof-k-armed-bai-fb-lowerbound}
Note that our proof follows the technique from~\cite{carpentier2016tight}. Let $\nu_k$ and $\nu'_k$ denote the Bernoulli distributions with parameters $p_k$ and $(1-p_k),$ respectively. Let $\nu_k^i$ denotes the distribution of arm~$k$ under $MAB_i.$ So, for $i = k,$ $\nu_k^i = \nu'_k,$ and for $i \neq k,$ $\nu^i_k = \nu_k.$
Define
\begin{align*}
d_k := 
\begin{cases}
1/2 - p_2 \hspace{1cm} \text{for $k =1$}\\
1/2 - p_k \hspace{1cm} \text{for $2 \le k \le K$}
\end{cases}
\end{align*}
Let  $\Delta^i_k$ denote the gap between weight-distorted rewards of the optimal arm ($= i$) and arm~$k$ under the problem transformation $MAB_i.$ This quantity is defined as follows: 
\begin{align*}
\Delta^i_k := 
\begin{cases}
(1-p_i) - p_k = (1/2 - p_i) + (1/2 - p_k) = d_i + d_k & \text{for $i \neq k.$}\\
1/2 - p_2 = d_1  & \text{for $i=k=1$} \\
(1 - p_i) -1/2 = d_i &\text{for $2 \le i=k \le K$}
\end{cases}
\end{align*}
Let $\widehat{KL}^i_{k,t}$ be the empirical KL-divergence between distributions of arm~$k$ in $MAB_i$ and $MAB_1,$ i.e.,
\begin{align}
\widehat{KL}^i_{k,t}  = \sum\limits_{s=1}^t \Big[ \mathbb{I} \lbrace X_{k,s} = 1\rbrace \log \left( \frac{\nu^1_k \left( X_{k,s} = 1 \right)}{\nu^i_k \left( X_{k,s} = 1 \right) } \right) + \mathbb{I} \lbrace X_{k,s} = 0 \rbrace \log \left( \frac{\nu^1_k \left( X_{k,s} = 0 \right)}{\nu^i_k \left( X_{k,s} = 0 \right) } \right) \Big].
\end{align}
Note that,
\begin{align*}
\widehat{KL}^i_{k,t} = 
\begin{cases}
0, \hspace{1cm} \text{for $i =1$ and $i \neq k$}\\
\sum\limits_{s=1}^t \Big[ \mathbb{I} \lbrace X_{k,s} = 1\rbrace \log \left( \frac{p_k}{1 - p_k } \right)  + \mathbb{I} \lbrace X_{k,s} = 0 \rbrace \log \left( \frac{1 - p_k}{p_k} \right) \Big], \hspace{0.25cm} \text{for $i = k$}
\end{cases}
\end{align*}
Define
\begin{align}
\label{eq:KL_k}
KL_k := KL(\nu_k, \nu'_k) &= KL (\nu'_k, \nu_k) = (1 - 2p_k) \log \left( \frac{1-p_k}{p_k} \right).
\end{align}
Let $\xi = \lbrace \forall 1 \le k \le K, \forall 1 \le m \le n, \widehat{KL}^k_{k,m} - mKL_k \le 2 \sqrt{m \log (6nK)} \rbrace.$
The following lemma shows that the empirical KL-divergences concentrate, implying $\xi$ is a high-probability event.
\begin{lemma}
\label{Lem:FBLB}
$\mathbb{P}_1 ( \xi ) \ge 5/6.$
\end{lemma}
\begin{proof}
It is easy to see that $\widehat{KL}^k_{k,m} - m KL_k \le 2 \sqrt{m \log (6nK)}$ trivially holds for $k = 1.$  Let $k \ge 2.$
\begin{align*}
\mathbb{E}_1 \left[ \widehat{KL}^k_{k,m} \right] & = \mathbb{E}_1 \left[ \sum\limits_{s=1}^m \left[ \mathbb{I} \lbrace X_{k,s} = 1\rbrace \log \left( \frac{p_k}{1 - p_k } \right) + \mathbb{I} \lbrace X_{k,s} = 0 \rbrace \log \left( \frac{1 - p_k}{p_k} \right) \right] \right]  \\ 
& = \sum\limits_{s=1}^m \left[ p_k \log \left( \frac{p_k}{1 - p_k } \right) + (1 - p_k) \log \left( \frac{1 - p_k}{p_k} \right) \right]\\
& =m KL_k
\end{align*} 
Recall that, $p_k \in [1/4, 1/2),$ for $k \ge 2.$ So, $2 \le k \le K, 1 \le s \le m,$ we get 
\begin{align}
\vert \mathbb{I} \lbrace X_{k,s} = 1\rbrace \log \left( \frac{p_k}{1-p_k} \right) + \mathbb{I} \lbrace X_{k,s}=0 \rbrace \log \left( \frac{1-p_k}{p_k} \right) \vert \le \log 3
\end{align}
Since $\widehat{KL}^k_{k,m}$ is an i.i.d. sum of bounded independent random variables, using Hoeffding's Inequality we obtain
\begin{align}
\mathbb{P}_1 \left( \widehat{KL}^k_{k,m} -m KL_k \le \sqrt{2} \log 3 \sqrt{m \log(6nK)} \right) \ge 1 - \frac{1}{6Kn}.
\end{align}
Since $\sqrt{2}\log 3 < 2,$ we obtain,
\begin{align*}
\mathbb{P}_1 \left( \widehat{KL}^k_{k,m} -m KL_k \le 2 \sqrt{m \log(6nK)} \right) &\ge 1 - \frac{1}{6Kn} \\
\mathbb{P}_1 \left( \exists 1\le k \le K, \exists 1 \le m \le n: \widehat{KL}^k_{k,m} -m KL_k \le 2 \sqrt{m \log(6nK)} \right) &\le \frac{1}{6} \\
\mathbb{P} \left( \forall 1 \le k \le K, \forall 1 \le m \le T, \widehat{KL}^k_{k,m} - mKL_k \le 2 \sqrt{m \log (6TK)} \right) &\ge 5/6,
\end{align*}
which completes the proof.
\end{proof}
Now we prove Theorem~\ref{Thm:FBSBLB1} with the aid of the above lemma.
\begin{proof}
Let $\varepsilon$ be any measurable event. Then, using change of measure inequality from~\cite{lai1985lowerbd}, we have
\begin{align}
\mathbb{P}_i (\varepsilon) = \mathbb{E}_1 \left[ \mathbb{I}_\varepsilon \exp \left( -\widehat{KL}^i_{i, T_i(n)} \right) \right].
\end{align} 
For $2 \le i \le K,$ define
\begin{align}
\varepsilon_i = \lbrace J_n =1 \rbrace \cap \lbrace \xi \rbrace \cap \lbrace T_i(n) \le 6 \mathbb{E}_1 [T_i(n)] \rbrace
\end{align}
\begin{align}
\mathbb{P}_i (\varepsilon_i) = \mathbb{E}_1 \left[ \mathbb{I}_{\varepsilon_i} \exp \left( -\widehat{KL}^i_{i, T_i(n)} \right) \right] &\overset{(a)}{\ge} \mathbb{E}_1 \left[ \mathbb{I}_{\varepsilon_i} \exp \left( -T_i(n) KL_i - 2 \sqrt{T_i \log (6nK)} \right) \right] \nonumber \\
& \overset{(b)}{\ge} \mathbb{E}_1 \left[ \mathbb{I}_{\varepsilon_i} \exp \left( -6 \mathbb{E}_1[T_i(n)] KL_i - 2 \sqrt{n \log (6nK)} \right) \right] \nonumber \\
& = \exp \left( -6 \mathbb{E}_1 [T_i(n)] KL_i - 2 \sqrt{n \log (6nK)} \right) \mathbb{P}_1( {\varepsilon_i}) \label{eq:kol1}
\end{align}
where $(a)$ is due to $\widehat{KL}^i_{i, T_i(n)} \ge -T_i(n) KL_i - 2 \sqrt{T_i(n) \log (6nK)}$ under $\varepsilon_i \subseteq \xi,$ $(b)$ is due to $T_i(n) \le 6 \mathbb{E}_1 [T_i(n)]$ and $T_i(n) \le n.$

Now, we bound $\mathbb{P}_1 (\varepsilon_i)$ as follows. Using the Markov inequality, we get
\begin{align}
\mathbb{P}_1(T_i(n) \ge 6\mathbb{E}_1 [T_i(n)] ) \le \frac{\mathbb{E}_1[T_i(n)]}{6 \mathbb{E}_1[T_i(n)]} = 1/6.
\end{align}
Assume that $\mathbb{P}_1 \left( J_n \neq 1 \right) \le 1/2.$ If this assumption is violated, then the algorithm's mistake probability on one of the $K$-armed bandit problems can be high, which makes the algorithm's mistake probability on $K$ bandit problems non-uniform. 
\begin{align*}
\mathbb{P}_1 \left( \lbrace J_n \neq 1 \rbrace \cup \lbrace \xi^c \rbrace \cup \lbrace T_i(n) \ge 6 \mathbb{E}_1[T_i(n)] \rbrace \right) \le 1/2 + 1/6 + 1/6 \Rightarrow \mathbb{P}_1 (\varepsilon_i) \ge 1/6.
\end{align*}
For $2 \le i \le K,$
\begin{align}
\mathbb{P}_i \left( J_n \neq i \right) = \sum\limits_{j \neq i} \mathbb{P}_i (J_n = j) & \ge \mathbb{P}_i (J_n =1 ) \ge \mathbb{P}_i (\varepsilon_i) \nonumber\\
& \overset{(a)}{\ge} \frac{1}{6} \exp \left( -6 \mathbb{E}_1 [T_i(n)] KL_i - 2 \sqrt{n \log (6nK)} \right) \label{eq:kol3},
\end{align}
where $(a)$ is due to~\eqref{eq:kol1}. Recall from~\eqref{eq:KL_k} and $p_i \in [1/4, 1/2)$ for $2 \le i \le K$,
\begin{align}
\label{eq:kol2}
KL_i = (1-2p_i) \log \left( \frac{1-p_i}{p_i} \right) \le 10 \left( 1/2 - p_i \right)^2.
\end{align}
As in the proof of Theorem \ref{Thm:FCSBLB}, we consider a weight function $w$ that satisfies \eqref{eq:specialw}.
 For this weight function, using~\eqref{eq:kol2}, we obtain
\begin{align*}
KL_i \le 10 \left(  1/2 - p_i \right)^2 \overset{(a)}{\le} \left( \frac{w(1/2) - w(p_i) }{L} \right)^{2/\alpha} \left( \frac{\Delta^1_i}{L} \right)^{2/\alpha},
\end{align*}
where $(a)$ follows from \eqref{eq:specialw}. From~\eqref{eq:kol3}, we have
\begin{align}
\label{eq:kol4}
\mathbb{P}_i (J_n \neq i) \ge  \frac{1}{6} \exp \left( -6 \mathbb{E}_1 [T_i(n)] \left( \frac{\Delta^1_i}{L} \right)^{2/\alpha} - 2 \sqrt{n \log (6nK)} \right).
\end{align}
Since $H_{1, \alpha} = \sum\limits_{2 \le k \le K} \left( \frac{L}{\Delta^1_k} \right)^{2/\alpha}$ and $\sum\limits_{1\le k \le K}\mathbb{E}_1[T_k(n)] = n$,  there exists $2 \le j \le K$ such that $\mathbb{E}_1[T_j(n)] \le \frac{n}{H_{1, \alpha} \left( \frac{\Delta^1_j}{L} \right)^{2/\alpha} }.$  For this $j$,~\eqref{eq:kol4} gives that 
\begin{align*}
\mathbb{P}_j (J_n \neq j) &\ge  \frac{1}{6} \exp \left( -\frac{60 n}{H_{1, \alpha}} - 2 \sqrt{n \log (6nK)} \right) \\
\Rightarrow \max\limits_{1 \le i \le K} \mathbb{P}_i (J_n \neq i) &\ge  \frac{1}{6} \exp \left( -\frac{60 n}{H_{1, \alpha}} - 2 \sqrt{n \log (6nK)} \right).
\end{align*}
This completes the proof.
\end{proof}

\subsection{Proof of Theorem~\ref{thm:bai_karmed_improved_lb}}
\label{proof:bai_karmed_improved_lb}
Proof of this theorem closely follows the proof of Theorem~1 in~\cite{carpentier2016tight}. Note that, the notation or quantities used in this proof are the same as the ones in the proof of Theorem~\ref{Thm:FBSBLB1} in our manuscript. From Theorem~\ref{Thm:FBSBLB1}, we have the following: 

Let $h^* = \sum\limits_{2 \leq  k \leq K} \frac{1}{d^{2/\alpha}_k H_{k,\alpha}}$ and it holds that $\sum\limits_{1\leq k \leq K} t_k = n$ then $\exists$ an $2 \leq i \leq K$ such that $t_i \leq \frac{n}{h^* d^{2/\alpha}_i H_{i, \alpha}}.$ For this $i,$ we get the following from equation~(42): 
\begin{align*}
\mathbb{P}_i \left( J_n \notin i \right) \geq \frac{1}{6} \exp \left( -\frac{60 n}{h^* H_{i, \alpha}} -2 \sqrt{n \log (6nK)} \right).
\end{align*} 

From the above and Theorem~\ref{Thm:FBSBLB1}, we have the following: 

If $n\geq\max\limits\left(H(1,\alpha),H_{i,\alpha}h^*\right)^24\log(6nK)/(60)^2$ then for any MAB algorithm that outputs an arm $J_n$ at time $n,$ it holds that 
\begin{align}
\label{eq:jnjd1}
\max\limits_{1 \leq i \leq K }\mathbb{P}_i \left( J_n \neq i \right)  \geq \frac{1}{6}	\exp \left( - 120 \frac{n}{H_{1, \alpha}} \right) 
\end{align}
and 
\begin{align}
\label{eq:jnjd2}
\max\limits_{1 \leq i \leq K } \left[ \mathbb{P}_i \left( J_n \neq i \right) \times \exp \left( 120 \frac{n}{H_{i, \alpha}h^*} \right) \right] \geq \frac{1}{6}	 
\end{align}

First part of theorem easily follows from equation~\eqref{eq:jnjd1}, since $H_{1, \alpha} = \max_i H_{i, \alpha}$ and $H_{i, \alpha}$ is bounded by $a.$

We establish the second part of the theorem by using the equation~\eqref{eq:jnjd2} as follows. Let $d_k = \frac{1}{4}\left(k/K\right)$ for $2 \leq k \leq K$ which implies that $p_k =1/2 - \frac{1}{4} \left( k/K \right) \in [1/4, 1/2),$ since $p_k = 1/2 - d_k.$ Note that, for this problem $H(i, \alpha) \leq H(1, \alpha) = \sum_{2\leq k \leq K} d_k^{-2/\alpha} \leq 3(4K)^{2/\alpha}$ and hence it belongs to $\mathbb{B}_a$ with $a = 3(4K)^{2/\alpha}.$ For this case, we now first calculate $d_i^{2/\alpha} H(i, \alpha)$ in order to calculate $h^*$. For any $1 \leq i \leq K,$ we have
\begin{align*}
d_i^{2/\alpha} H(i, \alpha) =  d_i^{2/\alpha} \sum\limits_{k \neq i} \frac{1}{(d_i + d_k)^{2/\alpha}} \leq d_i^{2/\alpha} \left( \frac{i}{d_i^{2/\alpha}} + \sum\limits_{k > i} \frac{1}{d_k^{2/\alpha}} \right) & \leq i + i^{2/\alpha}\sum\limits_{i < k \leq K} \frac{1}{k^{2/\alpha}} \\
& \leq i^{\frac{2}{\alpha}} + i^{\frac{2}{\alpha}}\sum\limits_{i < k \leq K} \frac{1}{k^{2}} \leq 4i^{2/\alpha}  
\end{align*}
From the above, we get that
\begin{align*}
h^* \geq \sum\limits_{i=2}^K \frac{1}{4i^{2/\alpha}} \geq \frac{1}{4\left( 1 - 2/\alpha \right)} \left[ (\log K)^{1-2/\alpha} - 2^{1-2/\alpha}  \right].
\end{align*} 
which establishes the theorem. 
\hfill{$\blacksquare$}

\section{Proofs for linear bandits setting}
\label{sec:proofs-linear}
Now, we present the detailed proofs of all technical results presented in Section~\ref{sec:algos-linear}. 

\subsection{Proofs for regret minimization setting}
\label{sec:proofs-regert-linear}
We establish a few technical results in the following lemmas that are necessary for the proof of Theorem \ref{thm:linear-bandit-regret}. 
The first step in the proof of Theorem \ref{thm:linear-bandit-regret} is to bound the instantaneous regret $ir_m$ at instant $m$, which is the difference in the weight-distorted rewards of arm $x_*$ and the arm $x_{m}$ chosen by the algorithm. For bounding the instantaneous regret, it is necessary to relate the difference in weight-distorted rewards of $x_*$ and $x_{m}$ to the difference in their means, i.e., $(x_*\tr\theta_* - x_{m}\tr\theta_*)$, and Lemma \ref{lemma:cptdiff} provides this connection.
\subsubsection{Proof of Lemma \ref{lemma:cptdiff}}
\label{sec:weight-distortion-undert-translation}
Notice that
	\begin{align*}
		\mu_{X+a} &= \intinfinity w(\prob{X+a > z}) dz - \intinfinity w(\prob{-X-a > z}) dz \\
		&= \intinfinity w(\prob{X > z-a}) dz - \intinfinity w(\prob{-X > z+a}) dz \\
		&= \int_{-a}^\infty w(\prob{X > u}) du - \int_{a}^\infty w(\prob{-X > u}) du \\
		&= \int_{-a}^0 w(\prob{X > u}) du - \int_{a}^0 w(\prob{-X > u}) du + \mu_X \\
		&= \int_{-a}^0 w(\prob{X > u}) du + \int_{-a}^0 w(\prob{X < v}) dv + \mu_X.
	\end{align*}
	This proves the main claim. For the case when $w$ is bounded by 1, it is easy to infer that $|\mu_{X+a} - \mu_X| \leq 2|a|$.
	\hfill{$\blacksquare$}
\begin{lemma}
\label{lemma:confidenceellipsoid}
Under the hypotheses of Theorem \ref{thm:linear-bandit-regret}, for any $\delta>0$, the following event occurs w.p. at least $1-\delta$: $\forall m \ge 1$, $\theta_*$ lies in the set $C_m$ defined in Algorithm \ref{alg:WOFUL}, i.e., 
\begin{align*}
	C_{m} &= \left\{\theta \in \mathbb{R}^d: \norm{\theta - \hat{\theta}_m}_{A_m} \leq D_m \right\}, \quad \mbox{where} \\
	D_m & = \sqrt{2 \log \left(\frac{\det(A_m)^{1/2} \det(\lambda I_{d \times d})^{1/2}}{\delta} \right)} + \beta\sqrt{\lambda}, \\
	A_m &= \lambda I_{d \times d} + \sum_{l=1}^{m-1} \frac{x_l x_l\tr}{\norm{x_l}^2}, \\
	\hat{\theta}_m &= A_m^{-1} b_m, \quad \mbox{and} \quad b_m = \sum_{l=1}^{m-1} \frac{r_l x_l}{\norm{x_l}}.
\end{align*}
\end{lemma}
\begin{proof}
The proof is a consequence of \cite[Theorem 2]{abbasi2011improved}. Indeed, the hypotheses of the stated theorem are satisfied with the sub-Gaussianity parameter $R = 1$ since by (\ref{eq:linban-depnoise}), the stochastic reward $r_m$ normalised by $\norm{x_m}$ at each instant satisfies
\[ \frac{r_m}{\norm{x_m}} = \frac{x_m}{\norm{x_m}} \tr \theta + \frac{x_m}{\norm{x_m}} \tr N_m, \]
whose distribution is Gaussian with mean $\frac{x_m}{\norm{x_m}} \tr \theta$ and variance $1$, and which is thus sub-Gaussian with parameter $R = 1$. 
\end{proof}

\subsubsection{Proof of Theorem \ref{thm:linear-bandit-regret}}
\label{thm:woful_linear_regret_bound_proof}
We first prove the following lemma which we will need in the proof of Theorem \ref{thm:linear-bandit-regret}. It relates the weight-distorted reward $\mu_x(\theta)$ of the stochastic reward from an arm $x \in \X$ to the weight-distorted reward of a translated standard normal distribution. 

\begin{lemma}[Weight-distorted value under scaling]
\label{lem:CPTscaling}
If $Z$ denotes a standard normal random variable, then for any arm $x$ and $\theta$, \[ \mu_{x}(\theta) = \norm{x} \cdot \mu_{Z + \frac{x\tr\theta}{\norm{x}}}, \]
where $\mu_{Y}$ denotes the weight-distorted reward of a random variable $Y$.
\end{lemma}

\begin{proof}
Let $X$ denote the stochastic reward from arm $x$ with $\theta$ as the parameter; we have that $X$ is normal with mean $x\tr\theta$ and variance $\norm{x}^2$. With $\stdnormal$ being a standard normal $d$-dimensional vector and $\hat x = \frac{x}{\norm{x}}$, we can write
	\begin{align*}  
	 \mu_X &:= \intinfinity w(\prob{X > z}) dz - \intinfinity w(\prob{-X > z}) dz\\
	 & =  \intinfinity w(\prob{x\tr\theta + x\tr\stdnormal > z}) dz - \intinfinity w(\prob{-x\tr\theta - x\tr\stdnormal > z})dz\\
	 & = \intinfinity w(\prob{\hat x\tr\theta + \hat x\tr\stdnormal > \frac{z}{\norm{x}}}) dz - \intinfinity w(\prob{-\hat x\tr\theta - \hat x\tr\stdnormal > \frac{z}{\norm{x}}})dz\\
	 &= \norm{x}\left(\intinfinity w(\prob{\hat x\tr\theta + \hat x\tr\stdnormal > y}) dy - \intinfinity w(\prob{-\hat x\tr\theta - \hat x\tr\stdnormal > y})dy\right)\\
	 &= \norm{x} \mu_{Z + \hat x\tr\theta};\stepcounter{equation}\tag{\theequation}\label{eq:t121}
	\end{align*}
\end{proof}

\begin{remark}\label{remark:sub-gauss-linear}
Note that the above proof holds even when $\stdnormal$ is a vector of sub-Gaussian random variables with parameter $1$ (cf. \cite[Chapter 2]{wainwright2019high}). Hence, the results for the linear bandit setting can be generalized to the case of sub-Gaussian noise.  
\end{remark}

\begin{proof}[Proof of Theorem \ref{thm:linear-bandit-regret}]
Let $ir_m = \mu_{x_*}(\theta_*) - \mu_{x_{m}}(\theta_*)$ denote the instantaneous regret incurred by the algorithm that chooses arm $x_{m}\in \X$ in round $m$.
Let $\arm_x$ (resp. $P_x$) denote the r.v. (resp. probability function) governing the rewards of the arm $x \in \X$.

Letting $\hat x_m = \frac{x_m}{\norm{x_m}}$ and $W$ to be a standard Gaussian r.v.,  we upper-bound the instantaneous regret $ir_m$ at round $m$ as follows: 
\begin{align}
ir_m =  \mu_{x_*}(\theta_*) - \mu_{x_{m}}(\theta_*) & \le \mu_{x_m}(\tilde\theta_m) - \mu_{x_{m}}(\theta_*) \label{eq:rm1}\\
& = \norm{x_m} \left( \mu_{W+ \hat x_{m}\tr\tilde\theta_m} - \mu_{W+ \hat x_{m}\tr\theta_*} \right) \label{eq:rm2}\\
& \le 2 \norm{x_m} \left| \hat x_{m}\tr(\tilde\theta_m - \theta_*) \right| \label{eq:rm3}\\
& = 2 \left| x_{m}\tr(\tilde\theta_m - \theta_*) \right| \label{eq:rm4}\\
& \le 2 \left(\left| x_{m}\tr(\hat\theta_{m} - \theta_*)\right|  + \left| x_{m}\tr(\tilde\theta_m - \hat\theta_{m})\right|\right)  \label{eq:rm5} \\
& = 2 \left(\norm{\hat\theta_{m} - \theta_*}_{A_m}  \norm{x_{m}}_{A_m^{-1}}  + \norm{\tilde\theta_m - \hat\theta_{m}}_{A_m}   \norm{x_{m}}_{A_m^{-1}} \right)  \label{eq:rm6} \\
& \le 4 w_m \sqrt{D_m}  \quad \text{whenever }\theta_* \in C_m, \label{eq:rm7}
\end{align} 
where $w_m = \sqrt{x_m\tr A_m^{-1}x_m}$ and $D_m$ is as defined in Algorithm \ref{alg:WOFUL}.
The inequalities above are derived as follows: \eqref{eq:rm1} follows from the choice of $x_{m}$ and $\tilde \theta_m$ in Algorithm~\ref{alg:WOFUL}; \eqref{eq:rm2} follows from the scaling lemma (Lemma \ref{lem:CPTscaling}); \eqref{eq:rm3} follows by applying the translation lemma (Lemma \ref{lemma:cptdiff}); \eqref{eq:rm4} follows from the definition of $\hat x_{m}$ \eqref{eq:rm6} is by the Cauchy-Schwarz inequality applied to the pair of dual norms $\norm{\cdot}_{A_m}$ and $\norm{\cdot}_{A_m^{-1}}$ on $\R^d$; \eqref{eq:rm7} follows from the definition of the confidence ellipsoid $C_m$ in Algorithm \ref{alg:WOFUL}, as $\hat\theta_{m}$, $\tilde \theta_m$ and $\theta_*$ belong to $C_m$. Note that Lemma \ref{lemma:confidenceellipsoid} guarantees that this is true w.p. at least $1-\delta$.

The cumulative regret $R_n$ of W-OFUL can now be bounded as follows. With probability at least $1-\delta$, 
\begin{align}
R_n = \sum_{m=1}^n ir_m \le  \sum_{m=1}^n 4 \sqrt{D_m} w_m &\le 4 \sqrt{D_n}  \sum_{m=1}^n w_m \nonumber \\
& \le 4 \sqrt{D_n}  \left(n \sum_{m=1}^n w_m^2\right)^{1/2} \label{eq:Rn3}\\
& \le \sqrt{32 d n D_n \log n}. \label{eq:Rn4}
\end{align} 
Inequality \eqref{eq:Rn3} follows by an application of the Cauchy-Schwarz inequality, while the inequality in \eqref{eq:Rn4} follows from (the standard by now) Lemma 9 in~\cite{dani2008stochastic}, which shows that $\sum_{m=1}^n w_m^2 \le  2d \log n$. 
The claim follows.
\end{proof}

\subsection{Proofs for best arm identification setting}
\subsubsection{Proof of Theorem~\ref{Thm:LinBdt1}.}
\label{sec:proof-W-G-upperbound}
We use the proof technique from~\cite{soare2014best} for proving this result. Since $\vert \mu_{x^*} \left( \hat{\theta}_n \right) - \mu_{x^*} \left( \theta^* \right) \vert \leq 2 \vert x^{*T} \left( \hat{\theta}_n - \theta^* \right) \vert$ (due to Lemma~\ref{lemma:cptdiff}), we get that
\begin{align*}
\left[  \mu_{x^*} \left( \hat{\theta}_n \right) - \mu_{x^*} \left( \theta^* \right) \right] - \left[  \mu_{x} \left(\hat{\theta}_n \right) - \mu_{x} \left( \theta^* \right) \right] &\geq -2 \vert x^{*T} \left( \hat{\theta}_n - \theta^* \right) \vert -2 \vert x^T \left( \hat{\theta}_n - \theta^* \right) \vert \\
& \overset{(a)}{\geq} -2 \Vert x^{*T} \Vert_{A_n^{-1}} \Vert \hat{\theta}_n - \theta^* \Vert_{A_n} -2 \Vert x^T \Vert_{A_n^{-1}} \Vert \hat{\theta}_n - \theta^* \Vert_{A_n} \\
&= -2 \Vert \hat{\theta}_n - \theta^* \Vert_{A_n} \left[ \Vert x^{*T} \Vert_{A_n^{-1}} +  \Vert x^T \Vert_{A_n^{-1}} \right],
\end{align*}
where $(a)$ is due to the Cauchy-Schwartz inequality. Hence, we have
\begin{align}
\hat{\Delta}_n \left( x^*, x \right) \geq \Delta \left( x^*, x \right) -2 \Vert \hat{\theta}_n - \theta^* \Vert_{A_n} \left[ \Vert x^{*T} \Vert_{A_n^{-1}} +  \Vert x^T \Vert_{A_n^{-1}} \right]
\end{align}
From Lemma~\ref{lemma:confidenceellipsoid}, we obtain
\begin{align}
\mathbb{P} \left( \forall n, \Vert \hat{\theta}_n - \theta^* \Vert_{A_n} \leq D_n \right) \geq (1 - \delta).
\end{align} 
Let $\rho^G_n = n \max\limits_{x \in \mathcal{X}} x^T A^{-1}_n x.$ Then, 
\begin{align}
\hat{\Delta}_n \left( x^*, x \right) \geq \Delta \left( x^*, x \right) -2 D_n \left[ \Vert x^{*T} \Vert_{A_n^{-1}} +  \Vert x^T \Vert_{A_n^{-1}} \right] &= \Delta \left( x^*, x \right) -2 D_n \sqrt{\frac{\rho^G_n}{n}}  \nonumber\\
& \geq \Delta_{\min} -2 D_n \sqrt{\frac{\rho^G_n}{n}} \label{eq:111}
\end{align}
From~\eqref{eq:Stop100}, we get the following sufficient condition for stopping:
\begin{align}
\label{eq:112}
\hat{\Delta}_n^2 (x^*, x) \geq \frac{D_n^2 \rho^G_n}{n} .
\end{align}
From~\eqref{eq:111} and \eqref{eq:112}, we get the following sufficient condition:
\begin{align}
\Delta_{\min} -2 D_n \sqrt{\frac{\rho^G_n}{n}} \geq D_n \sqrt{\frac{ \rho^G_n}{n}}.
\end{align}
Or equivalently,
\begin{align}
\frac{ 9 D_n^2 \rho^G_n}{n}\leq \Delta_{\min}^2.
\end{align}
It is easy to see that for $n = \frac{ 9 D_n^2 \rho^G_n}{\Delta_{\min}^2}$, the above inequality is satisfied. Hence, the sample complexity $N^G$ of W-G algorithm satisfies
\begin{align*}
\mathbb{P} \left( N^G \leq \frac{ 9 D_n^2 \rho^G_n}{\Delta_{\min}^2} \,\, \text{and returning optimal arm} \right) \geq (1 - \delta).
\end{align*}
Note that for large $n$ and large number of arms, the optimization problem in \eqref{eq:arm-selection-cpt-g-1} becomes an NP-hard problem~\cite{soare2014best}. However, a wide number of approximate methods have been proposed in the literature to solve the same~\cite{sagnol2013approximation, soare2014best}. Using~\cite{soare2014best}, we say that the performance of any $\epsilon$-approximate method to solve~\eqref{eq:arm-selection-cpt-g-1} is no worse than a factor of $(1+\epsilon)$ of the solution obtained by solving the incremental version of ~\eqref{eq:arm-selection-cpt-g-1} given below:
\begin{align*}
\label{eq:arm-selection-cpt-g-2}
x_m= \arg\min\limits_{x' \in \mathcal{X} }  \max\limits_{x \in \mathcal{X}} \Vert x \Vert_{\left( A_{\bf{a_{m-1}}} + x' x'^T \right)^{-1}} .
\end{align*}
Hence, from the proof of Theorem~1 in~\cite{soare2014best}, we get that $\rho^G_n \leq (1+\epsilon) d,$ where $\epsilon$ appears due to the $\epsilon$-approximation method used to solve the NP-hard optimization problem in~\eqref{eq:arm-selection-cpt-g-1}. Hence, we have that
\begin{align*}
\mathbb{P} \left( N^G \leq \frac{ c D_n^2 d}{\Delta_{\min}^2} \,\, \text{and returning optimal arm} \right) \geq (1 - \delta),
\end{align*}
where $c = 9 \left( 1 + \epsilon \right)$ and this completes the proof.
\hfill{$\blacksquare$}

\section{Numerical experiments for regret minimization under $K$-armed bandits}
\label{sec:experiments-regret-appendix}
We describe two sets of experiments pertaining to the $K$-armed settings. For both problems, we take the $S$-shaped weight function $w(p) = \frac{p^\eta}{\left(p^\eta + (1-p)^\eta \right)^{1/\eta}}$, with $\eta = 0.61$ and H\"{o}lder parameters as $\alpha = \eta$, $L = 1/\eta$ to model perceived distortions in cost/reward. It has been observed in \cite{tversky1992advances} that this distortion weight function is a good
fit to explain distorted value preferences among human beings.

We first study three stylized $2$-armed bandit problems which, in part, draw upon experiments carried out by~\cite{tversky1992advances} on human subjects in their studies of cumulative prospect theory. Figures \ref{fig:karmed} (a) and (b) describe each problem setting in detail (following the convention of this paper of modeling costs, ``\$$x$ w.p. $p$" is taken to mean a loss of \$$x$ suffered with a probability of $p$.)

For the first problem, the weight function $w$ gives the distorted cost of arm 1 as \$$10.55$, much higher than its expected cost of \$$5$ and thus more expensive due to the deterministic Arm 2 with a (distorted and expected) cost of \$$7$. The distortion in costs thus shifts the optimal arm from Arm 1 to Arm 2, and an online learning algorithm must be aware of this effect in order to attain low regret with respect to choosing Arm 2. A similar pattern is true for the other problem involving truly stochastic arms -- weight distortion favors the arm with the higher cost in true expectation. 

We benchmark the cumulative regret of two algorithms -- (a) the well-known UCB algorithm \cite{auer2002finite}, and (b) W-UCB; the results are as in Figure \ref{fig:karmed}. In the experiments, UCB is not aware of the distorted weighting and hence suffers linear regret with respect to playing the optimal distorted arm, due to converging to essentially a `wrong' arm. On the other hand, the W-UCB algorithm, being designed to explicitly account for distorted cost perception, estimates the distortion using sample-based quantiles and exhibits significantly lower regret. 

\begin{figure}
\centering
  \begin{tabular}{cc}
  \hspace{-1em}
  \begin{subfigure}{0.5\textwidth}
  \label{fig:a}
 \tabl{c}{\scalebox{0.75}{\begin{tikzpicture}
      \begin{axis}[
	xlabel={Rounds},
	ylabel={Cumulative regret},
       clip mode=individual,grid,grid style={gray!30}]
      \addplot+[error bars/.cd,y dir=both,y explicit, every nth mark=10000] table [x=X,y=Y,y error=Y_error,col sep=comma] {karmed-UCB-Setting1.txt};
      \addplot+[error bars/.cd,y dir=both,y explicit, every nth mark=10000] table [x=X,y=Y,y error=Y_error,col sep=comma] {karmed-WUCB-Setting1.txt};
      \legend{UCB,W-UCB}
      \end{axis}
      \end{tikzpicture}}\\}
  \end{subfigure}
  &
  \hspace{-1em}
  \begin{subfigure}{0.5\textwidth}
  \label{fig:b}
   \tabl{c}{\scalebox{0.75}{\begin{tikzpicture}
      \begin{axis}[
	xlabel={Rounds},
	ylabel={Cumulative regret},
       clip mode=individual,grid,grid style={gray!30}]
      \addplot+[error bars/.cd,y dir=both,y explicit, every nth mark=10000] table [x=X,y=Y,y error=Y_error,col sep=comma] {karmed-UCB-Setting2.txt};
      \addplot+[error bars/.cd,y dir=both,y explicit, every nth mark=10000] table [x=X,y=Y,y error=Y_error,col sep=comma] {karmed-WUCB-Setting2.txt};
      \legend{UCB,W-UCB}
      \end{axis}
      \end{tikzpicture}}\\}

  \end{subfigure}
  \end{tabular}
\caption{Cumulative regret along with standard error from $100$ replications for the UCB and W-UCB algorithms for 2 stochastic $K$-armed bandit environments: 
(a) {\em Arm 1:} (\$50 w.p. 0.1, \$$0$ w.p. 0.9) {\bf vs.} {\em Arm 2:} \$7 w.p. 1. (b) {\em Arm 1:} (\$500 w.p. 0.01, \$$0$ w.p. 0.99) {\bf vs.} {\em Arm 2:} (\$250 w.p. 0.03, \$$0$ w.p. 0.97).%
}
\label{fig:karmed}
\end{figure}

In the following set up, we use a heuristic estimator for weight-distorted reward as $W$(empirical mean) for distributions with support $\subseteq$ $[0, 1].$  Note that our weight-distorted value estimate turns out to be $W(\textrm{empirical mean})$ for the Bernoulli distributions. We adapt the classic UCB~\cite{auer2002finite} policy to this set up by taking the above \textbf{W}(\textbf{E}mpirical mean) as a surrogate for the weight-distorted reward estimate and call the policy WE-UCB.   

\textbf{Setup(*).} We consider a simple two-armed stochastic bandit with arm distributions' as follows: 
\[
\text{Arm-1} =
\begin{cases}
  1  & \text{w.p.} \quad  0.1 \\
  0.2  & \text{w.p.} \quad 0.1 \\
  0  & \text{w.p.} \quad 0.8
\end{cases}
\quad \quad
\text{Arm-2} =
\begin{cases}
  1  & \text{w.p.} \quad  10^{-3} \\
  0.2  & \text{w.p.} \quad 0.8 \\
  0  & \text{w.p.} \quad 0.199
\end{cases}
\]
For convenience, we treat these arm distributions as corresponding to rewards and hence the optimal arm is defined as the arm with highest weight-distorted reward. In the following, we show the expected and weight-distorted values of both arms: 
\begin{table}[ht]
\centering 
\begin{tabular}{c c c} 
\hline 
Distribution & Expected value & Weight-distorted value \\ [0.5ex] 
\hline 
Arm-1 & 0.12 & 0.2005 \\ 
Arm-2 & 0.161 & 0.1327\\ [1ex] 
\hline 
\end{tabular}
\label{table:new_setup2} 
\end{table}

Thus, arm-2 is optimal with respect to expected value and arm-1 is optimal with respect to weight-distorted reward where the latter is calculated empirically using a very large number of i.i.d. samples. We now present the comparison of UCB, WE-UCB and W-UCB for a fixed time horizon in Figure~\ref{fig:setup_2_regret}. We fix the time horizon as $10^4$ and run the algorithms for 100 runs and averaged them to obtain the results. It is clear from Figure~\ref{fig:setup_2_regret} that our algorithm W-UCB (achieves logarithmic regret) outperforms both UCB and WE-UCB (both achieve linear regrets) in this simple setup. Hence, these results justify the need of our algorithms under the weight-distorted setup considered herein.   

\begin{figure}[htb]
\center{\includegraphics[scale=0.6]{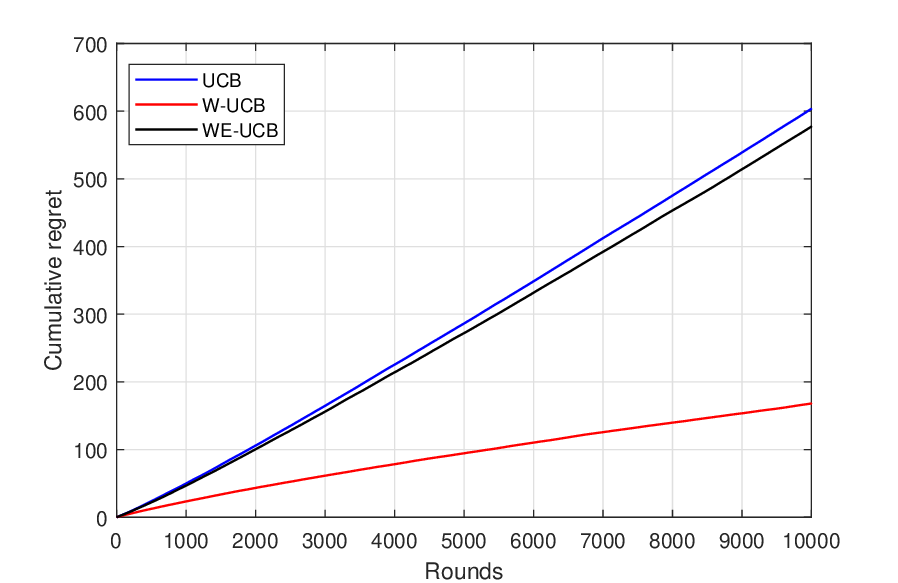}}
\caption{\label{fig:setup_2_regret} Comparison of UCB, WE-UCB and W-UCB for regret minimisation in the stochastic bandits.}
\end{figure}

\section{Numerical experiments for best arm identification setting}
\label{sec:experiments-bestarm}
\label{sec-experiments}
\textit{Basic methodology.} The simulations presented in this section are performed using MATLAB, and all results are averaged over 100 independent runs. The expected value for each arm is obtained analytically, while the weight-distorted reward is estimated numerically through Algorithm \ref{alg:cpt-estimator} using $50000$ samples from the underlying distribution of the individual arms.
\subsection{$K$-armed bandits}
We consider the following three-armed bandit problem to test the fixed confidence and budget algorithms: Arm~1 takes values $50$ and $0$ with probabilities $0.1$ and $0.9$, respectively; Arm~2 takes values $30$ and $0$ with probabilities $0.2$ and $0.8$; and  Arm~3 takes value $7$ with probability $1$. Under the weight function  $w(x) = \frac{x^{0.61}}{\left( x^{0.61} + (1-x)^{0.61} \right)^\frac{1}{0.61}}$, arm~1 has the highest weight-distorted reward. On the other hand, the expected value of arm~3 is the highest. We implement LUCB \cite{kalyanakrishnan2012pac}, SR \cite{audibert2010best} and their variants that incorporate weight-distorted reward. The latter algorithms are W-LUCB and W-SR, described in Section \ref{sec:algos-karmed-bestarm}. 

We ran SR and our W-SR algorithm for simulation budgets ranging from $10^4$ samples to $5 \times 10^4$ samples on the aforementioned three-armed bandit problem. We observed that SR and W-SR converge to arm 3 and arm 1, respectively, whose expected and weight-distorted rewards values are listed in Table~\ref{tab:sr-results}. It is evident that a specialized algorithm such as W-SR is necessary to optimize weight-distortion-based criteria, and the regular SR algorithm that optimizes expected value cannot be used as a surrogate to find a weight-distortion optimal solution.         
\begin{table}
  \caption{Expected and weight-distorted rewards for a three-armed bandit problem. W-SR converges to arm $1$, while SR converges to arm $3$; the expected/weight-distorted rewards of these arms are tabulated below.}
  \label{tab:sr-results}
 \centering
 \begin{tabular}{c|c|c}
  \toprule 
   & \textbf{Expected reward}& \textbf{Weight-distorted reward }\\\midrule
   SR & \textbf{$7$} & $7$ \\\midrule
   W-SR & $5$ & \textbf{$9.304$}\\\bottomrule
  \end{tabular}
  \end{table}
Next, we compare the performance of LUCB and our W-LUCB algorithm for various confidence parameters and present the results in Figure~\ref{Fig:FC-1}. For a given confidence parameter, the empirical probability of incorrect identification of an algorithm is calculated by the fraction of runs on which the algorithm returned a sub-optimal arm. It is easy to see that W-LUCB's empirical probability of incorrect identification is always lower than the given confidence parameters, while that of LUCB is above the given confidence parameters. Thus, it is very likely that LUCB returned a sub-optimal arm, while W-LUCB found the optimal arm. As in the case of fixed confidence algorithms, the results in Figure \ref{Fig:FC-1} justify the need for an algorithm like W-LUCB, as LUCB does not find the weight-distortion optimal solution.
\begin{figure}[h]
\centering
\begin{tikzpicture}
\begin{axis}[ylabel style={align=center}, xlabel={{\small Expected sample complexity}}, ylabel={{\small Confidence/probability of}\\ {\small incorrect identification}},width=7cm,height=6cm,grid,tick scale binop={\times},
legend style={font=\small},
legend style={at={(0,0)},anchor=north,at={(axis description cs:0.5,-0.26)},legend columns=2},
x label style={at={(axis description cs:0.5,-0.0025)},anchor=north},
every x tick scale label/.style={
    at={(1,0)},xshift=-1pt,yshift=-1.2em,anchor=south west,inner sep=0pt
}
]
\addplot+[only marks,mark=o,mark options={scale=1.2}] file {results/expected_sample_complexity_LUCB_delta.txt};
\addlegendentry{Confidence (LUCB)}
\addplot+[only marks,mark=square*,mark options={scale=1.2}] file {results/expected_sample_complexity_LUCB_proberror.txt};
\addlegendentry{Probability of incorrect identification (LUCB)}
\addplot+[only marks,mark=o,mark options={scale=1.2}] file {results/expected_sample_complexity_CPTLUCB_delta.txt};
\addlegendentry{Confidence (W-LUCB)}
\addplot+[only marks,mark=square*,color=green,mark options={scale=1.2}] file {results/expected_sample_complexity_CPTLUCB_proberror.txt};
\addlegendentry{Probability of incorrect identification (W-LUCB)}
\end{axis}
\end{tikzpicture}
\caption{Comparison of LUCB and W-LUCB algorithms in fixed confidence setting using the probability of incorrect identification in identifying the arm with the highest weight-distorted reward.}
\label{Fig:FC-1}
\end{figure}  

We now compare the performance of LUCB, W-LUCB and WE-LUCB, where the latter is an adaptation of classic LUCB to fixed confidence setting by using $W(\textrm{empirical mean})$ as a surrogate for the weight-distorted value estimate. We use the same Setup(*) that is considered in Section~\ref{sec:experiments-regret-appendix} here. Figure~\ref{fig:setup_2} shows the comparison of LUCB, WE-LUCB and our W-LUCB for various confidence ($\delta$) parameters. We have run the algorithms for 100 runs and calculated expected sample complexity as the average sample complexity across all the runs. Similarly, probability of incorrect identification is calculated as the fraction of runs where the algorithm has returned a sub-optimal arm. It is easy to observe from Figure~\ref{fig:setup_2} that our algorithm W-LUCB outperforms both LUCB and WE-LUCB in this very simple set-up.

\begin{figure}[htb]
\center{\includegraphics[scale=0.6]{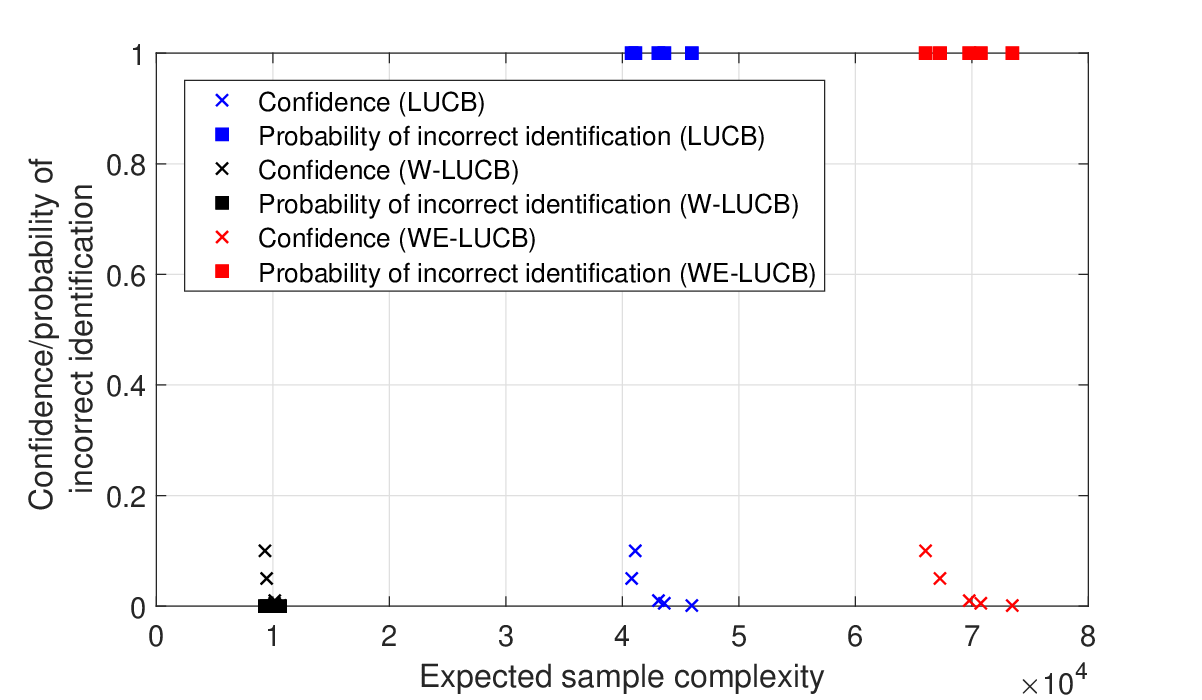}}
\caption{\label{fig:setup_2} Comparison of LUCB, WE-LUCB and W-LUCB for best arm identification in stochastic bandits with fixed confidence setting.}
\end{figure}

\subsection{Linear bandits}  

\begin{figure}
\centering
\scalebox{0.75}{
 \begin{tikzpicture}[font=\sffamily, every matrix/.style={ampersand replacement=\&,column sep=2cm,row sep=2cm}, process/.style={draw,thick,circle,fill=blue!20}]
\matrix{
\node[process] (a1) {1}; \&    \node[process] (a2) {2};
      \& \node[process] (a3) {3}; \& \\

   \node[process] (a4) {4};
      \& \node[process] (a5) {5};   \& \node[process] (a6) {6};\\
   \node[process] (a7) {7};
      \& \node[process] (a8) {8};   \& \node[process] (a9) {9};\\
  };
 \draw[green!40!black] (a1) --node[above]{$l_1$} (a2);
 \draw[green!40!black] (a2) --node[above]{$l_2$} (a3);
 \draw[green!40!black] (a4) --node[above]{$l_6$} (a5);
 \draw[green!40!black] (a5) --node[above]{$l_7$} (a6);
 \draw[green!40!black] (a7) --node[above]{$l_{11}$} (a8);
 \draw[green!40!black] (a8) --node[above]{$l_{12}$} (a9);
 \draw[green!40!black] (a1) --node[right]{$l_3$} (a4);
 \draw[green!40!black] (a2) --node[right]{$l_4$} (a5);
 \draw[green!40!black] (a3) --node[right]{$l_5$} (a6);
 \draw[green!40!black] (a4) --node[right]{$l_8$} (a7);
 \draw[green!40!black] (a5) --node[right]{$l_9$} (a8);
 \draw[green!40!black] (a6) --node[right]{$l_{10}$} (a9);
%
\end{tikzpicture}}
\caption{A $3\times3$-grid network}
\label{Fig:grid-network}
\end{figure}
As shown in Figure~\ref{Fig:grid-network}, we consider a $3\times3$-grid network with nodes labelled $1$ and $9$ as the source and destination, respectively. Each link $l_i$ is associated with a parameter $\theta_i$ that indicates the mean delay for the link. We consider the following six paths as the arms of the linear bandit: $x_1 = (l_1, l_2, l_5, l_{10}),$ $x_2 = (l_1, l_4, l_9, l_{12}),$ $x_3 = (l_3, l_8, l_{11}, l_{12}),$ $x_4 = (l_3, l_6, l_7, l_{10}),$ $x_5 = (l_3, l_6, l_9, l_{12}),$ $x_6 = (l_1, l_4, l_7, l_{10}).$ We set $\theta = \left[ 30\,10\,30\,5\,1\,20\,15\,7\,20\,30\,1\,30 \right]$, leading to mean delays of $71, 85, 69, 95, 100, 80$ for arms $1$ to $6$, respectively. 

A road user traversing any path will experience a sample delay which is the sum of mean overall  delay together with a noise element that captures the stochastic nature of the routing problem.  The noise element uses the following model: arms $1, 2, 4, 5$ and $6$ have standard Gaussian noise, while the noise element for arm 3 takes values $100$ and $0$ with probabilities $0.01$ and $0.99$, respectively. Thus, arm $3$ has the lowest expected delay.

    \begin{table}
	\caption{Expected delay and weight-distorted rewards for the linear bandit problem corresponding to 3$\times3$ grid network shown in Figure~3. W-G converges to arm $1$, while G-allocation converges to arm $3;$ the expected delay and weight-distorted rewards of these arms are tabulated below.}
	\label{tab:cpt-g-results}
	\centering
	\begin{tabular}{c|c|c}
		\toprule 
		& \textbf{Expected delay}& \textbf{Weight-distorted reward}\\\midrule
		G-allocation & \textbf{$69$} & $7.795$ \\\midrule
		W-G & $71$ & \textbf{$8.865$}\\\bottomrule
	\end{tabular}
\end{table} 

For the weight-distorted reward, we used a baseline delay of $80$, corresponding to arm $6$. The rationale is that road users currently take arm $6$ and any routing improvement is relative to the status quo (or) if a new path results in a overall delay of $70$, then it corresponds to a gain of $10$, while a delay of $100$ corresponds to a loss of $20$. The goal is to maximize the weight-distorted reward of the delay gain, where the gain is defined relative to baseline (arm $6$). In this setting, arm $1$ has the highest weight-distorted reward.
Intuitively, arm $3$ is less appealing to a human road user as it has a small probability of resulting in a very high delay, even though its mean delay is the lowest. On the other hand, arm $1$ is slightly sub-optimal in the mean delay, but has no chance of inordinately high delays.
We compare G-allocation \cite{soare2014best} that attempts to find a path with the lowest mean delay against our W-G algorithm that incorporates weight-distortion based criteria to find a path that is more appealing to a human road user.   

 Table \ref{tab:cpt-g-results} presents the results from simulation runs of W-G and G-allocation, with  $\delta = 0.01$. As the reader might expect, W-G algorithm converges to arm $1$, while G-allocation in~\cite{soare2014best} to arm $3$.  Therefore,  W-G algorithm is more conducive in identifying a path that emulates human decision making.

\end{document}